\documentclass{article}

\usepackage{microtype}
\usepackage{graphicx}
\usepackage{subfigure}
\usepackage{booktabs} 

\usepackage{hyperref}

\usepackage[accepted]{icml2023}

\usepackage{amsmath}
\usepackage{amssymb}
\usepackage{mathtools}
\usepackage{amsthm}

\newcommand\at[2]{\left.#1\right|_{#2}}
\usepackage{macros}
\usepackage{multirow}
\usepackage{xcolor}
\usepackage{siunitx,etoolbox}
\usepackage{centernot}
\usepackage{nicefrac}       

\theoremstyle{plain}
\newtheorem{theorem}{Theorem}[section]
\newtheorem{proposition}[theorem]{Proposition}
\newtheorem{lemma}[theorem]{Lemma}

\theoremstyle{definition}

\theoremstyle{remark}

\usepackage[textsize=tiny]{todonotes}

\icmltitlerunning{Adaptive Annealed Importance Sampling with Constant Rate Progress}

\begin{document}

\twocolumn[
\icmltitle{Adaptive Annealed Importance Sampling with Constant Rate Progress}



\icmlsetsymbol{equal}{*}

\begin{icmlauthorlist}
\icmlauthor{Shirin Goshtasbpour}{eth,sdsc}
\icmlauthor{Victor Cohen}{sdsc}
\icmlauthor{Fernando Perez-Cruz}{eth,sdsc}
\end{icmlauthorlist}

\icmlaffiliation{eth}{Computer Science Department, ETH Zurich, Zurich, Switzerland}
\icmlaffiliation{sdsc}{Swiss Data Science Center, Zurich, Switzerland}

\icmlcorrespondingauthor{Shirin Goshtasbpour}{shirin.goshtasbpour@inf.ethz.ch}

\icmlkeywords{Normalization Factor Estimation, Annealed Importance Sampling}

\vskip 0.3in
]



\printAffiliationsAndNotice{\icmlEqualContribution} 

\begin{abstract}
    Annealed Importance Sampling (AIS) synthesizes weighted samples from an intractable distribution given its unnormalized density function. This algorithm relies on a sequence of interpolating distributions bridging the target to an initial tractable distribution such as the well-known geometric mean path of unnormalized distributions which is assumed to be suboptimal in general. In this paper, we prove that the geometric annealing corresponds to the distribution path that minimizes the KL divergence between the current particle distribution and the desired target when the feasible change in the particle distribution is constrained. Following this observation, we derive the constant rate discretization schedule for this annealing sequence, which adjusts the schedule to the difficulty of moving samples between the initial and the target distributions. We further extend our results to $f$-divergences and present the respective dynamics of annealing sequences based on which we propose the Constant Rate AIS (CR-AIS) algorithm and its efficient implementation for $\alpha$-divergences. We empirically show that CR-AIS performs well on multiple benchmark distributions while avoiding the computationally expensive tuning loop in existing Adaptive AIS.

\end{abstract}

\section{Introduction}
\label{sec:intro}

Annealed Importance Sampling (AIS) \citep{neal2001annealed} is one of the most popular sampling methods to estimate intractable expectations given an unnormalized density of a distribution. Together with its other variants such as thermodynamic integration \citep{ogata1989monte, gelman1998simulating} and Sequential Monte Carlo (SMC) \citep{del2006sequential} this algorithm has vast applications, such as marginal likelihood estimation \citep{salakhutdinov2008quantitative, grosse2015sandwiching, grosse2016measuring}, moment estimation \citep{johansen2015towards, jasra2011inference}, generative model evaluation \citep{wu2017on} and more recently it has been incorporated in variational inference and training of deep generative networks \citep{maddison2017filtering, naesseth2018variational, wu2020stochastic,thin2021monte, masrani2019thermodynamic}.

To perform annealing, this algorithm uses a sequence of bridging distributions between proposal distribution and the target which is chosen in advance. \citet{gelman1998simulating} have demonstrated that the optimal path with lowest variance for thermodynamic integral estimator depends on the Hellinger distance of the distributions and it is intractable in complex setups. Instead, the geometric mean path has been utilized for years \citep{neal2001annealed, neal1996sampling}. As alternatives, moment-averaging have been proposed for exponential family distributions \citep{grosse2013annealing} and further generalized to power mean for arbitrary endpoint distributions \citep{masrani2021q, brekelmans2020annealed}. These heuristic annealing paths achieve viable estimation results even though they are considered to be suboptimal. 

In this work, we analyze a version of AIS algorithm where we apply infinitesimal changes to the initial density along the annealing distribution path to get to the target distribution. We take a greedy approach and modify the particle distribution in the direction that optimally reduces the remaining estimation bias at every instance. The remaining bias is equivalent to the inverse KL-divergence between current particle distribution and the target distribution under common assumptions \cite{grosse2013annealing} and we prove that in this setup the optimal greedy strategy is achieved using the geometric mean path. Extending our analysis to the larger class of $f$-divergences, we are able to derive an Ordinary Differential Equation (ODE) for the optimal greedy annealing dynamics. In the subclass of $\alpha$-divergences, this ODE has a closed form solution and we show that power mean annealing is a solution to this equation. 

While other variational representations of geometric and power mean paths have been provided in the previous work \citep{grosse2013annealing, masrani2021q}, to the best of our knowledge, we are the first to show the relation of these annealing sequences with the functional derivatives of various probability divergences. Using this framework, we are able to derive the constant rate schedule along the steepest descent annealing path. 

Our greedy strategy is similar to existing Adaptive AIS in that we look ahead and measure the impact of each annealing step base on an objective function. However, in a typical Adaptive AIS algorithm, the schedule is adjusted in each step to keep the reduction of the Effective Sample Size (ESS) \cite{kong1992note,neal2001annealed} or the Conditional ESS (CESS) \cite{johansen2015towards} at a constant rate using iterative search algorithms. Instead, we derive the annealing distribution path and its corresponding schedule using the same objective derivative. Therefore, the constant rate schedule is tightly connected to the bridging distributions and is able to account for the difficulty of synthesis along each annealing sequence. 

Finally, we design the Constant Rate AIS (CR-AIS) algorithm to approximate the constant rate schedule of the variational objectives. We present multiple considerations for its practical implementation. CR-AIS does not rely on searching algorithms and excessive target density function evaluations as in adaptive versions of AIS and uses the information from the derivative of the objective to choose the bridging distributions. Using this algorithm we empirically verify our findings on high dimensional targets and illustrate how CR-AIS is able to trade-off computation complexity with estimation accuracy while improving adaptivity.

\section{Annealed Importance Sampling}
\label{sec:aisbg}

Formally, suppose $P$ and $Q_0$ are two probability distributions on $\R^d$ with density functions $\pi$ and $q_0$, respectively. We assume evaluation of $q_0$ is tractable while $\pi=\tilde \pi / Z_\pi$ is only known up to the normalization constant $Z_\pi = \int_{\R^d} \tilde \pi(z)dz$. To sample from $\pi$, AIS uses a sequence of annealing distributions defined by the density path $\gamma : [0,1]\times \R^d \to \R_+$ where $\gamma(t, \cdot) \in\mathcal P$ for all $t\in[0,1]$ starting from $\gamma(0,\cdot) = q_0(\cdot)$ and reaching to $\gamma(1,\cdot) = \pi(\cdot)$ and $\mathcal P$ is the family of normalized density functions. This path is discretized with the schedule $0=t_0 < ... < t_M=1$. Common choices for schedule are linear discretization with $t_i = i / M$, exponential with $t_i = 1 - \eps^i$ and sigmoidal with $t_i = \sigma(c(i/M -0.5))$ for hyperparameters $\eps < 1$, $0 < c$ and $\sigma(x) = 1/(1 + e^{-x})$. 

A Markov process $\qr(z_{0:M}) = q_0(z_0)\prod_{i\in[M]}\qr_i(z_i|z_{i-1})$ is used for sampling such that particles sampled from the initial distribution, $z_0\sim q_0$, gradually move following each transition probability $\qr_i$ to have a marginal distribution close to $\gamma(t_i,\cdot)$.  An auxiliary backward Markov chain $\ql_i(z_{i-1}|z_i)$ allows us to compute the unnormalized importance weights corresponding to each particle trajectory $z_{0:M}$,
\begin{align}
    w(z_{0:M}) = \frac{\tilde\pi(z_M) \prod_i \ql_i(z_{i-1}|z_i)}{q_0(z_0)\prod_i\qr_i(z_i|z_{i-1})}.\nonumber
\end{align}
It is common to define the transition probabilities $\qr_i$ to be reversible with respect to $\gamma(t_i,\cdot)$ (e.g. a Markov chain with Metropolis-Hastings corrected transition kernels) and select $\ql_i$ as its reversal \citep{neal2001annealed}. Let us denote the unnormalized density path with $\tilde \gamma(t, \cdot) = Z_{\gamma_{t}}\cdot\gamma(t, \cdot)$, where $Z_{\gamma_t} = \int_{\R^d} \tilde \gamma(t, z)dz$. Therefore, with reversible transitions, $\ql_i(z_{i-1}|z_i) = \qr_i(z_i|z_{i-1})\tilde \gamma(t_i,z_{i-1}) / \tilde \gamma(t_i, z_i)$, we can rewrite the importance weights as $w(z_{0:M}) = \prod_i \tilde \gamma(t_i, z_{i-1})/\tilde \gamma(t_{i-1}, z_{i-1})$.
The unbiased Monte Carlo (MC) estimator of the partition function 
from $N$ sampled particle instances $z_{0:M}^{(j)}\sim q_0\prod_i \qr_i$ for $1\leq j \leq N$ is 
\begin{align}
    Z_\pi \approx \hat Z_\pi=\frac{1}{N} \sum_j w(z_{0:M}^{(j)}). \nonumber
\end{align}
Although, to avoid numerical underflow, log space computations are preferred and $\log Z_\pi$ is bounded from below by \citep{grosse2015sandwiching, domke2018importance}
\begin{align}
    \log Z_\pi \geq &\mathbb E_{\qr}[\log w(z_{0:M})]\label{eq:Zelbo}\\
    &\approx\frac{1}{N}\sum_j \log w(z_{0:M}^{(j)}).\nonumber
\end{align}
It is possible to estimate the expectation of a test function $h:\R^d\to\R$ under the target distribution via self-normalized weighted average of $h(z_M^{(j)})$ or equivalently by resampling the particles according to a multinomial distribution where $z_M^{(j)}$ is sampled proportional to $w(z_{0:M}^{(j)})$.

\section{Adaptive Annealing Dynamics from Divergence Derivative}
\label{sec:annealingdyn}

Given the particle trajectory, $z_{0:M}\sim \qr$, we denote the density of the marginal distribution of $z_i$ with $q_{t_i}(z_i)$. It is common to analyze the AIS algorithm in perfect transition regime (i.e. when $q_{t_i}(z) = \gamma(t_i,z)$ $q_{t_i}$-a.s.) with reversible transition kernels \cite{grosse2013annealing,kiwaki2015variational}. This is not  unrealistic if $t_i - t_{i-1}$ is sufficiently small and the consecutive annealing distributions are close to each other. We assume the same conditions apply in our paper. Under this setup, \citet{grosse2013annealing} decomposed the bias of the estimator in \Cref{eq:Zelbo} as the sum of KL divergences between consecutive annealing distributions \citep{grosse2013annealing},
\begin{align}\label{eq:cost}
    \log Z_\pi - \mathbb E_{\qr}[\log w(z_{0:M})] = \sum_{i=1}^M \text{D}_\text{KL}(q_{t_{i-1}} || q_{t_i}),
\end{align}
where $\text{D}_\text{KL}(q||p) = \int q(z) \log (q(z)/p(z))dz$.
In fact, as $M\to\infty$ the asymptotic bias decreases as 
\begin{align}\label{eq:asym}
    M \sum_{i=1}^M \text{D}_\text{KL}\left(q_{t_{i-1}}||q_{t_i}\right) \to \frac{1}{2}\int_0^1 \text{Var}_{q_t}[\frac{d}{dt}\log q_t]dt.
\end{align}

The end-to-end asymptotic bias in \Cref{eq:asym} was used to compare the efficiency of moment-averaging and geometric mean paths in \citep{grosse2013annealing} and optimized with respect to the single dimensional schedule function for a given density path $\tilde \gamma$ in \citep{kiwaki2015variational}. Their method does not scale to higher dimensional function space to optimize the annealing density path. Instead, we propose an adaptive approach to maximize the reduction in bias with every infinitesimal transition and we use the derivative of KL-divergence at the current particle distribution to find the optimal change in the annealing density path. By doing so, we are able to extend the optimization space from the space of discretization schedules to the space of unnormalized annealing density paths. In the following, we explain the details of our method. Proof of the results are provided in \Cref{sec:app:proofs}.

\subsection{Inverse KL Divergence Dynamics}
\label{sec:kldyn}

Let $\tilde q_t(z) = \tilde \gamma(t,z)$ be the unnormalized marginal density at instant $t$ and $J_\text{KL}[\phi_t]$ define the functional of $\phi_t(z) = \log \tilde q_t(z)$ corresponding to the inverse KL divergence,
\begin{align}
    J_\text{KL}[\phi_t] = &\text{D}_\text{KL}(q_t||\pi) \label{eq:kl}\\
    = &\frac{\int \tilde q_t(z) \log \frac{\tilde q_t(z)}{\tilde \pi(z)}dz}{\int \tilde q_t(z)dz}    &+ \log Z_\pi - \log \int \tilde q_t(z)dz.\nonumber
\end{align}

In our greedy strategy, at step $i$, we fix all the previous annealing distributions up to $q_{t_i}$ and consider none of the subsequent ones other than $q_1=\pi$. We choose the next distribution $q_{t_{i+1}}$ as an infinitesimal modification of $q_{t_i}$ which minimizes the updated sampling bias. We can derive the updated bias recursively from \Cref{eq:cost} starting from $b_0 = \text{D}_\text{KL}(q_0||\pi)$ and repeating 
\begin{align}\label{eq:bi}
    b_i=&b_{i-1} + \text{D}_\text{KL}(q_{t_i}||q_{t_{i+1}}) \nonumber\\
    &+ \text{D}_\text{KL}(q_{t_{i+1}}||\pi) - \text{D}_\text{KL}(q_{t_i}||\pi),
\end{align}
for $i < M$. To find $q_{t_{i+1}}$ we take the directional functional derivative of $b_i$ and minimize it in a compact space to find the steepest descent direction that leads to the optimal annealing. As the first and last terms in \Cref{eq:bi} are constant with respect to $q_{t_{i+1}}$ and the derivative of the second term is zero, the directional derivative is equivalent to the directional derivative of $J_\text{KL}[\phi_{t_i}]$. 
First, we derive the directional derivative of $J_\text{KL}$ in the following Lemma. Then, in \Cref{theo:geom-steepest}, we show that the geometric path,
\begin{align}\label{eq:geom}
    \log \tilde q_t^\text{geom}(z) \coloneqq (1 - t) \log q_0(z) + t \log \tilde \pi(z),
\end{align}
corresponds to the path following this direction and is optimal in this sense. 

To take the derivative of $J_\text{KL}$, we transform the negative energy $\phi_{t}$ by a small perturbation in direction of the smooth function $\eta:\R^d\to\R$ with a small step size $\eps>0$. 
\begin{lemma}
\label{lem:kl-derivative}
    Assume $\tilde q_t(z)$ and $\tilde \pi(z)$ are positive unnormalized density functions and let $\phi_{t+\eps}(z) = \phi_t(z) + \eps \eta(z)$ for $\phi_t(z) = \log \tilde q_t(z)$. Then we have,
    \begin{align}\label{eq:kl-derivative}
        \at{\frac{d}{d\eps}J_\textup{KL}[\phi_{t+\eps}]}{\eps=0} = &\textup{Cov}_{q_t}\left[\eta(z), \log \frac{\tilde q_t(z)}{\tilde \pi(z)}\right],
    \end{align}
    where $\textup{Cov}_q[\cdot,\cdot]$ is the covariance under distribution of $q$ and we use the definition of G\^{a}teaux differential for the derivative,
    \begin{align}
        \at{\frac{d}{d\eps}J_\textup{KL}[\phi_{t+\eps}]}{\eps=0} = &\lim_{\eps \to 0^+} \frac{J_\textup{KL}[\phi_t(z) + \eps \eta(z)] - J_\textup{KL}[\phi_t]}{\eps}. \nonumber
    \end{align}
\end{lemma}

To identify the optimal perturbation direction, $\eta^*_{q_t,\pi}$, we minimize $\text{Cov}_{q_t}\left[\eta(z),\log \frac{\tilde q_t(z)}{\tilde \pi(z)}\right]$ with respect to $\eta$ in the space of smooth functions with bounded variance. Using a bound on the variance of the perturbation as opposed to its norm is explained by the fact that the expectation of the perturbation only impacts the normalization factor of the bridging density function and does not affect the performance of AIS. Therefore, by constraining the perturbations to have bounded variance, we are accounting for the equivalency of annealing paths with different time dependent scaling. 

The optimal perturbation $\eta^*_{q_t,\pi}$ is the steepest descent direction of the inverse KL divergence which we can use to derive the annealing dynamics via
\begin{align}\label{eq:f-ode}
    \frac{d}{dt}\log \tilde q(z) &= \eta_{q_t,\pi}^*(z) + b(t).
\end{align} 

Note that $b(t):\R\to\R$ in \Cref{eq:f-ode} is an arbitrary log scale function of $t$ which can be absorbed by $\eta^*_{q_t,\pi}(z)$ for each $t$ without loss of generality. In the following, we show that with initial distribution density $\tilde q_0$ and following the infinitesimal perturbations in direction of $\eta_{q_t,\pi}^*$
we recover an arbitrarily scaled geometric mean path. In addition to the dynamics of the optimal greedy annealing path, we derive a discretization schedule in the following proposition which ensures a steady decrease in $J_\text{KL}$ as AIS algorithm proceeds.
\begin{proposition}
    \label{theo:geom-steepest} 
    Assume the same conditions as in \Cref{lem:kl-derivative}. Additionally, consider the set of smooth perturbation directions with bounded variance \begin{align}
        \mathcal M_{q_t,\pi}\coloneqq \{\eta\in \mathcal C^1: \textup{Var}_{q_t}[\eta(z)]\leq c_{q_t,\pi}^\text{KL}\},\nonumber
    \end{align}
    for $B \geq 0$ and $c_{q_t,\pi}^\text{KL} = B/\textup{Var}_{q_t}[\log (\tilde \pi(z)/\tilde q_t(z))]$.
    Then the steepest descent direction that minimizes the derivative in \Cref{eq:kl-derivative} in $\mathcal M_{q_t,\pi}$ is
    \begin{align}\label{eq:etaopt-kl}
        \eta^*_{q_t,\pi}(z) = \frac {c_{q_t,\pi}^\text{KL}}{\sqrt B} \log\frac{\tilde\pi(z)}{\tilde q_t(z)} + b,
    \end{align}
    for arbitrary $b\in\R$. A solution to the Ordinary Differential Equation (ODE) $\frac{d}{dt}\phi_t(z) = \eta^*_{q_t, \pi}(z)$ with initial condition $\phi_0(z) = \log \tilde q_0(z)$ is the scaled geometric mean path $\log \tilde q^\text{geom}_{1 - \beta(t)}(z)$ which for $\beta(t)$ set as 
    \begin{align}\label{eq:betageom}
        \beta^\text{KL}(t) \coloneqq e^{-\int_0^tc_{q_r,\pi}^\text{KL}dr/\sqrt B},
    \end{align}
decreases the inverse KL divergence in \Cref{eq:kl} with constant rate.
\end{proposition}

Note that the annealing process described by \Cref{theo:geom-steepest} encompasses the duration $0 < t$ as opposed to $t\in[0,1]$ in the original AIS algorithm. This is due to taking infinitesimal annealing steps along the derivatives of $J_\text{KL}$. In particular, the annealing schedule in \Cref{theo:geom-steepest} is defined with $\tau(t) = 1 - \beta(t)$ and for $\beta(t)=\beta^\text{KL}(t)$ as $t\to \infty$, $\tau(t)\to 1^-$ and $q_t(z) \to \pi(z)$ for all $z\in\R^d$, although the normalization factor of $\tilde q_t$ may grow to infinity. Additionally, due to the derivative of $\tau(t)$,
\begin{align}
    \dot\tau(t) \propto (1 - \tau(t))/\text{Var}_{q_t}[\log (\tilde \pi/\tilde q_t)],\nonumber
\end{align}
annealing is slower when the particle distribution is further away from the target (i.e. in the beginning of the annealing process). The constant rate schedule enforces a balanced division of the sampling difficulty per annealing step in comparison to the heuristics such as linear. In other words, to converge to the target, in each iteration, the particle distribution is altered by doing an equal amount of work which is measured in terms of the reduction of the KL-divergence. 

Due to the logarithmic term in the derivative, annealing is slower when the particle distribution extends beyond the target distribution support, while being less sensitive to unexplored modes in the target distribution with a damped sensitivity as t grows. However, in \Cref{sec:experiment1} we show this not to be an issue as using this schedule results in coverage of all of the modes of multimodal targets in the experiments.

\subsection{Extension to $f$-Divergences}

In this section, we extend our method to $f$-divergences to explore other dynamics that are optimal with respect to alternative step-wise objectives. The $f$-divergence between two distributions with unnormalized densities $\tilde \pi(z)$ and $\tilde q_t(z)$ is defined as
\begin{align}
    J_f[\phi_t] = \text{D}_f(\pi||q_t) = \frac{\int \tilde q_t(z) f \left(u_t(z)\right)dz}{\int \tilde q_t(z)dz},\nonumber
\end{align}
where $f:\R\to\R$ is convex, lower-semicontinuous function with $f(1) = 0$ and 
\begin{align}
    u_t(z) = \frac{\tilde \pi(z)}{\tilde q_t(z)}/\mathbb E_{q_t}\left[\frac{\tilde \pi(z)}{\tilde q_t(z)}\right].\nonumber
\end{align}

Following a similar approach to the previous section, we can find the optimal perturbation of $\phi_t$ and use \Cref{eq:f-ode} to obtain the annealing unnormalized density dynamics along the steepest descent direction. In \Cref{lem:f-derivative} we derive the steepest descent direction of $f$-divergence.

\begin{lemma}
\label{lem:f-derivative}
    Assume $\tilde q_t(z)$ and $\tilde\pi(z)$ are positive unnormalized density functions and $f:\R\to\R$ be convex and differentiable. Let $\phi_{t+\Eps(z)} = \phi_t(z) + \eps \eta(z)$ for $\phi_t(z) = \log \tilde q_t(z)$. Then we have,
    \begin{align}
        \at{\frac{d}{d\eps}J_f[\phi_{t+\eps}]}{\eps=0} = &\textup{Cov}_{q_t}\left[\eta(z), -g(u_t(z))\right],\nonumber
    \end{align}
    where $g(u) = u \dot f(u) - f(u)$ and $\dot f(u) = df(u)/du$. Moreover, consider the set of smooth perturbation directions with bounded variance $\mathcal M^f_{q_t,\pi}\coloneqq \left\{\eta\in \mathcal C^1: \textup{Var}_{q_t}[\eta(z)]\leq c_{q_t,\pi}^f\right\}$ for $B \geq 0$ and
    \begin{align}
        c_{q_t,\pi}^f = B/\textup{Var}_{q_t}\left[g(u_t(z))\right].\nonumber
    \end{align}
    Then the steepest descent direction that minimizes this derivative in $\mathcal M_{q_t,\pi}^f$ is
    \begin{align}
        \eta^*_{q_t,\pi}(z) = \frac{c_{q_t,\pi}^f}{\sqrt B}g(u_t(z)) + b,\nonumber
    \end{align}
    for arbitrary $b\in\R$.
\end{lemma}

Unfortunately, the solution of ODE in \Cref{eq:f-ode} with the optimal perturbation direction does not have a closed form for general $f$ functions. However, we can get a set of solutions for the specific case of $\alpha$-divergences in the following Proposition, which correspond to the power mean annealing path previously proposed in \citep{brekelmans2020annealed}\footnote{For illustrations of bridging distributions with different $\alpha$ values we refer the readers to the works of \citet{brekelmans2020annealed} and \citet{masrani2021q}},
\begin{align}\label{eq:power}
    \tilde q^{\alpha\text{-pow}}_t(z) = \left(t\tilde \pi(z)^\alpha + (1 - t)\tilde q_0(z)^\alpha\right)^{\frac{1}{\alpha}}.
\end{align}

\begin{proposition}
    \label{theo:power-steepest} 
    Assume the same conditions as in \Cref{lem:f-derivative} and $f(u) = (u^\alpha - 1 - \alpha(u-1))/\alpha(\alpha-1)$ for $\alpha \centernot\in \{0,1\}$ or $f(u)=u\log u$ for $\alpha=1$. Then $\alpha$-power mean path $\log \tilde q_{1-\beta(t)}^{\alpha\text{-pow}}$ is a solution to the ODE $\frac{d}{dt}\phi_t(z) = \eta^*_{q_t, \pi}(z)$ with initial condition $\phi_0(z) = \log \tilde q_0(z)$ and setting $\beta(t)$ to
    \begin{align}\label{eq:beta}
        \beta^\alpha(t) \coloneqq e^{-\int_0^t (c_{q_r,\pi}^fZ_{q_r}^\alpha/\sqrt BZ_\pi^\alpha)dr},
    \end{align}
    results in constant rate decrease in $f$-divergence.
\end{proposition}

Here, the annealing speed is inversely related to $\Var_{q_t}[g(\pi/q_t)]$. With $\alpha=1$ the annealing dynamics correspond to the arithmetic mean path (moment-averaging path in the exponential family \citep{grosse2013annealing}). We have listed a few of the popular choices of $\alpha$-divergences with their respective $f$ and $g$ functions in \Cref{tab:alpha}.

\begin{table}[h]
    \caption{List of $f$-divergences}\label{tab:alpha}
    \begin{center}
    \begin{small}
    \begin{sc}
    \setlength\tabcolsep{2.5pt}
    \vskip .2in
    \begin{tabular}{|cccc|}
    \hline
        Name & $f(u)$ & $g(u)$ & $\alpha$\\
    \hline
        KL divergence & $u\log u$ & $u$ & $1^*$\\
        Inverse KL divergence & $-\log u$ & $\log u - 1$ & $0^*$ \\
        Pearson $\chi^2$ & $(u - 1)^2$ & $u^2 - 1$ & $2$\\
        Neyman $\chi^2$ & $\frac{(u - 1)^2}{u}$ & $2 - \frac{2}{u}$ & $-1$\\
        Squared Hellinger & $(\sqrt u  - 1)^2$ & $\sqrt u - 1$ & $1/2$\\
        $\alpha$-div ($\alpha\centernot\in\{0,1\}$) & $\frac{u^\alpha -\alpha u}{\alpha(\alpha-1)} + \frac{1}{\alpha}$ & $\frac{1}{\alpha}(u^\alpha + 1)$ & $\alpha$ \\
    \hline
    \end{tabular}
    \end{sc}
    \end{small}
    \end{center}
\end{table}


\section{Constant Rate AIS}

Using \Cref{theo:power-steepest}, we can design an adaptive AIS algorithm with constant rate decrease in $\alpha$-divergence at each annealing iteration. 
In  \Cref{algo}, we provide a pseudocode for Constant Rate AIS (CR-AIS) where we alternate between updating the particle location with standard AIS steps and tuning the schedule. Evaluation of the constant rate schedule $\tau(t) = 1 - \beta^\alpha(t)$ given in \Cref{eq:beta} at time $t$ requires the values of $c_{q_r,\pi}^f$ and $Z_\pi/Z_{q_r}$ for all $0 \leq r \leq t$ and integration. We use weighted AIS particles up to step $i$ to estimate the integrand in $\beta^\alpha(t_i)$. We set $t_i = i\delta$ for small $\delta > 0$ and approximate the integral with Riemann sum
\begin{align}
    \int_0^{i\delta} (c_{q_r,\pi}^fZ_{q_r}^\alpha/\sqrt BZ_\pi^\alpha)dr\approx \sum_{k=0}^{i} \delta c_{q_{k\delta},\pi}^fZ_{q_{k\delta}}^\alpha/\sqrt BZ_\pi^\alpha,\nonumber
\end{align}
where $q_{k\delta}=q^{\alpha\text{-pow}}_{1-\beta^\alpha(k\delta)}$ is the normalized density of power mean path in \Cref{eq:power}. We set $B=1$ for simplicity. We can rewrite the above equation incrementally, noting that
\begin{align}\label{eq:betarec}
    \beta^\alpha((i+1)\delta) &= \beta^\alpha(i\delta)e^{- \int_{i\delta}^{(i+1)\delta} (c_{q_r,\pi}^fZ_{q_r}^\alpha/Z_\pi^\alpha)dr}\nonumber\\
    &\approx \beta^\alpha(i\delta) \exp(- \delta c_{q_{i\delta},\pi}^fZ_{q_{i\delta}}^\alpha/Z_\pi^\alpha).
\end{align}

In the \Cref{algo}, 
To compute $\exp(- \delta c_{q_{i\delta},\pi}^fZ_{q_{i\delta}}^\alpha/Z_\pi^\alpha)$ we use particles $z_{0:i}\sim q_0\prod_{k=1}^i \qr_k$ and their weights $w(z_{0:i}) = \prod_{k=1}^i \tilde q_{t_k}(z_{k-1})/\tilde q_{t_{k-1}}(z_{k-1})$ given by the AIS algorithm up to iteration $i$ to estimate $Z_{q_{i\delta}}/Z_\pi$ (line 7-9) and the empirical variance under $q_{i\delta}$ (line 10-11). We reduce the variance of the estimated integrand by reusing the same set of particles to perform both estimations for all transitions. Having these estimates, we can approximate $\beta^\alpha((i+1)\delta)$ recursively from approximation of $\beta^\alpha(i\delta)$ in the previous iteration using \Cref{eq:betarec} (line 12). 

Using the constant rate schedule estimate, the next annealing density is updated (line 13-14) and particles $(z_i^j)_j$ are transitioned to their new location with a transition kernel which is invariant with respect to $q_{(i+1)\delta}$ (line 16-17).  We update the importance weights with standard AIS  procedure (line 18) and repeat the process until convergence. 

To ensure stability of the numerical computations we abort the algorithm when the empirical variance becomes lower than a given threshold. This condition indicates that the last annealing density is sufficiently close to the target. However, the empirical variance given by the particles may be much lower than the true value and mislead the algorithm to terminate the annealing process in a handful of steps especially for larger $|\alpha|$. We recommend to mitigate this problem by constraining the maximum step size in the schedule as we did in our experiments with high dimensional targets.

Another consideration is to use disjoint sets of particles for adjusting the schedule and testing. We recommend to perform the final estimation after the tuning phase and fixing the schedule to ensure the independence between samples and the consistency of the importance weighted estimate.

When \Cref{eq:f-ode} does not have a closed form solution, we can use CR-AIS with numerical approximation of its solution instead (line 14 of \Cref{algo}). New annealing paths may result in more robust estimation in practice as we can optimize function $f$ more effectively. We leave an efficient implementation of this extension to the future work.

\begin{algorithm}[tb]
    \caption{CR-AIS tuning for $\alpha$-divergences}\label{algo}
    \begin{algorithmic}[1]
    \STATE {\bfseries Input:} Target $\tilde\pi$, proposal density $q_0$, $\alpha$, $\delta$
    \STATE {\bfseries Output:} Discretization sequence $(\tau_{i\delta})_i$

    \STATE Set $i \gets 0$, $\beta_0 \gets 1$ and $\tau_0 \gets 0$
    \STATE Draw $z_0^j\sim q_0(z)$ for $j\in[N]$ and concatenate into $\mathbf z_i$.
    \STATE Set $\text{logw}^j \gets -\log q_0(z_0^j)$.
    \WHILE{ not converged and $i < \text{max\_steps}$}
        \STATE Set $\log \hat Z_\pi \gets \text{logsumexp}(\text{logw} + \log \tilde \pi(\mathbf z_i))$.
        \STATE Set $\log \hat Z_{q_{i\delta}} \gets \text{logsumexp}(\text{logw} + \log \tilde q_{i\delta}(\mathbf z_i))$.
        \STATE Set $r_i \gets \exp(\log \hat Z_\pi - \log \hat Z_{q_{i\delta}})$.
        \STATE Set $u_i^j \gets \frac{\tilde \pi(z^j_i)}{r_i\tilde q_{i\delta}(z^j_i)}$ for $j\in[N]$ and concat. to $\mathbf u_i$.
        \STATE Set $v_i\gets \widehat{\text{Var}}_{q_{i\delta}}[g(\mathbf u_i)]$ .
        \STATE Set $\beta_{(i+1)\delta} \gets \beta_{i\delta}\exp(-\delta /v_ir_i^\alpha)$.
        \STATE Set $\tau_{(i+1)\delta} = 1 - \beta_{(i+1)\delta}$.
        \STATE Set $\tilde q_{(i+1)\delta} = \tilde q_{\tau_{(i+1)\delta}}^{\alpha\text{-pow}}$ from \Cref{eq:power}.
        \STATE Set $i \gets i + 1$.
        \STATE Construct a $\qr_i(z|z_{i-1}^j)$ invariant w.r.t. $q_{i\delta}$.
        \STATE Draw $z_i^j\sim \qr_i(z|z_{i-1}^j)$ for $j\in[N]$.
        \STATE Set $\text{logw}^j \gets \text{logw}^j + \log \frac{\ql_i(z^j_{i-1}|z^j_i)}{\qr_i(z^j_i|z^j_{i-1})}$.
    \ENDWHILE
    \STATE Set $\text{logw}^j \gets \text{logw}^j + \log \tilde \pi(z_i^j)$.
    \end{algorithmic}
\end{algorithm}

\section{Experiments}

\begin{figure*}[h]
\vskip 0.2in
    \begin{center}
        \begin{minipage}[b]{1.\textwidth}
        \centerline{\includegraphics[width=1.\textwidth]{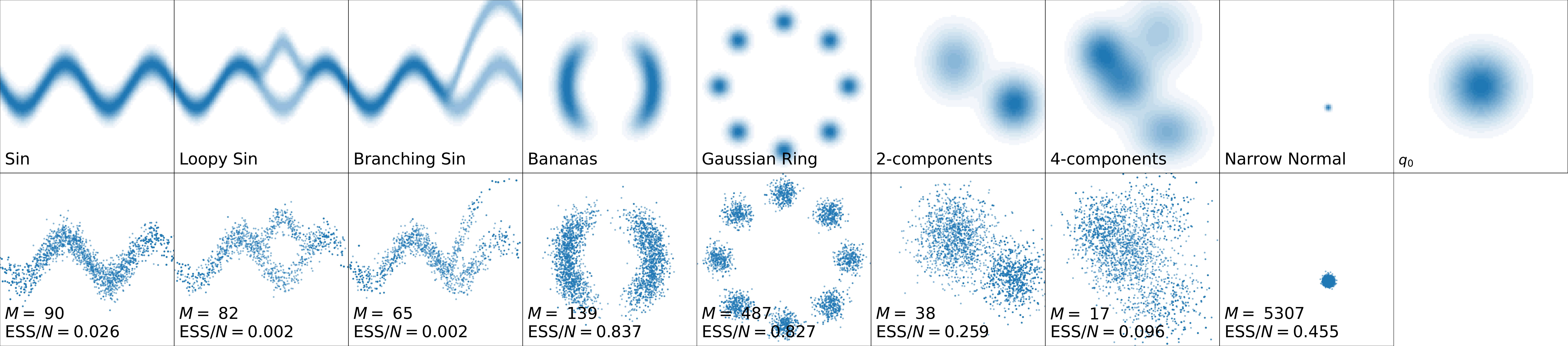}}
        \end{minipage}
        \caption{Target distributions and resampled particles with CR-AIS, with geometric mean path and $\delta = 1/32$. CR-AIS adjusts the number of iterations $M$ to the difficulty of the target and covers the support of the targets. The plot in top right corner is the initial distribution used for annealing with the same scale.}
    \label{fig:particles}
    \end{center}
\vskip -0.2in
\end{figure*}

In this section, we run a number of experiments to illustrate the performance of CR-AIS on support coverage and adaptivity with 2d distributions and we asses its efficiency and accuracy with estimation of the log normalization constant of high dimensional synthetic targets and the posterior of  Bayesian models. Code is available at \url{https://github.com/shgoshtasb/cr\_ais}.

\subsection{2D Distribution Synthesis}\label{sec:experiment1}

Here, we investigate the adaptability of our algorithm to various target distributions. We evaluate CR-AIS on complex 2d distributions which are often used to benchmark sampling \citep{rezende2014stochastic} and three other distributions. With 2D targets, we can plot the synthesized samples and assess if indeed the particle distribution converges to the targets. We use the inverse KL objective ($\alpha=0$), initial distribution $\mathcal N(0,I)$ and set $\delta=1/32$ for all the targets as we aim to show the performance differences caused by the difficulty of the sampling task (setup details in \Cref{sec:app:2d}).

\Cref{fig:particles} depicts the target distributions next to resampled particles according to the importance weights of the CR-AIS algorithm. The initial distribution is shown in top right corner of \Cref{fig:particles} with the same scale. Each plot is annotated with the number of AIS iterations, $M$, and the ESS ratio to number of particles. The algorithm adapts the number of iterations to the difficulty of sampling from the target distribution with the Narrow Gaussian distribution requiring the longest annealing sequence. Additionally, the plots show that particles have reached all the modes of the targets and cover their support. In \Cref{sec:app:smc} we compare the SMC variant of our algorithm and adaptive SMC on wider multimodal target distributions. Our similar results confirm superiority of the constant rate schedule in mode coverage despite using shorter annealing sequences.

\subsection{Adjusting the Schedule to the Annealing Path}\label{sec:experiment2}

In a setup similar to the previous experiment, we investigate how the constant rate schedule adapts across different target distributions and annealing paths. We plot the emerging approximated schedule in CR-AIS for four of the previous targets: the very narrow Gaussian target, Gaussian ring, the dual mode Bananas and the Gaussian mixture distribution with 4 components. We vary $\alpha$ between $\{-0.5, 0, 0.5, 1, 2\}$ to obtain different annealing paths. We compare the schedules to the ones obtained from Adaptive AIS where the schedule is adjusted to decrease CESS at an approximate rate of 0.7. 
We constrain the maximum step size in Adaptive AIS since large steps cause severe weight degeneracy and premature termination of annealing.
\begin{figure*}[h]
\vskip 0.2in
    \centering
        \centerline{\includegraphics[width=1.\textwidth]{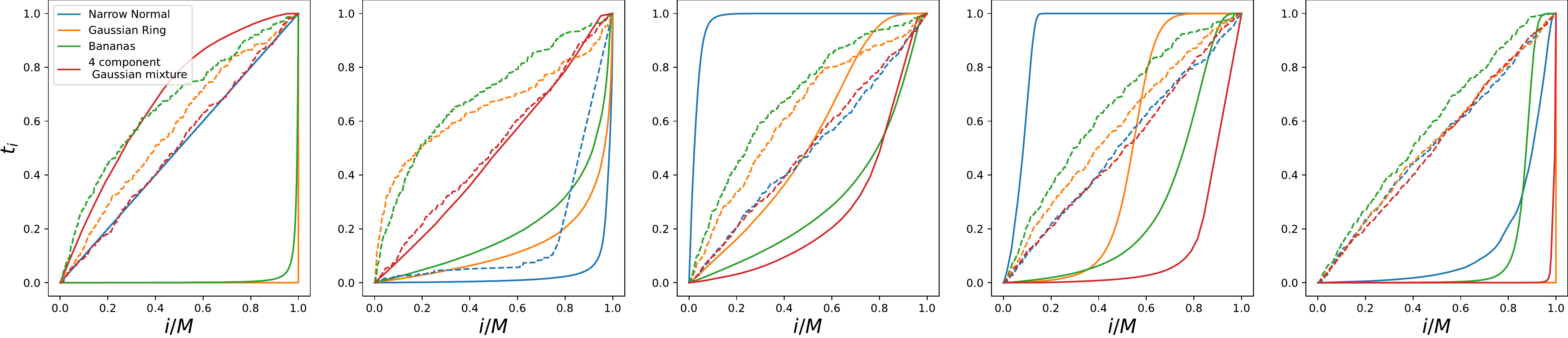}}
        \caption{Annealing schedules of CR-AIS (solid curves) and Adaptive AIS (dashed curves) across different annealing paths ($\alpha$ equal to $-0.5, 0, 0.5, 1$ and 2 from left to right) and targets, $\delta=1/32$. CR-AIS shows higher adaptivity while Adaptive AIS is essentially indifferent to the annealing path.}
        \label{fig:alphabeta}
\vskip -0.2in
\end{figure*}

In \Cref{fig:alphabeta}, we show the constant rate schedule $(\tau(i\delta))_i$ emerging from CR-AIS (solid curves) and Adaptive AIS schedules (dashed curves) for different targets and bridging distributions.
The constant rate schedule varies considerably between targets and depends explicitly on the similarity of the target distribution and the particle distribution at each time. Consider the second plot from left with $\alpha=0$. As mentioned before, when majority of particles are in the regions with small $\pi$ annealing slows down significantly, e.g. in the beginning of annealing for the narrow Gaussian example. It also explains why a much larger number of iterations are required to sample from this target ($M=5307$) while for the others the number of iterations remain moderate (487, 139 and 17). 
In contrast, Adaptive AIS has close to linear schedule on the geometric path for 3 of the distributions as CESS depends on weighted averages which may be misleading due to the weight degeneracy problem. For the narrow Gaussian where the target and initial distributions are far apart, Adaptive AIS recovers a schedule similar to CR-AIS. 

Across different values of $\alpha$, the constant rate schedule gradually changes its form depending on the target distribution. For Gaussian mixture, which has overlapping high density regions with the initial distribution, the initial annealing speed reduces monotonically with $\alpha$. For Gaussian Ring and Bananas, which have modes outside the typical region of $q_0$, the initial speed grows from $\alpha=0$ to $\alpha=0.5$ and decreases for larger $\alpha$ values. Whereas, Adaptive AIS schedule is relatively indifferent to the changes in the annealing path and has to infer the path characteristics through CESS. The high flexibility may lead to stability problems with large variances (e.g. for Gaussian Ring and $\alpha=-0.5$ or $\alpha=2$) where the schedule grows very slowly and the algorithm reaches the maximum number of iteration, or when the variance is  underestimated (e.g. for the Normal distribution with $\alpha=-0.5$) leading to a large step to the target. To avoid these pathologies we recommend to constrain the minimum and maximum of step size as in Adaptive AIS.

\subsection{Estimation of $\log Z_\pi$ in High Dimensions}\label{sec:experiment3}

We explore the absolute error of log normalization factor estimation for simple $d=128$ and $512$ dimensional distributions.
For the following targets: narrow Gaussian $\mathcal N(0, 0.01I)$, a mixture of 8 Gaussian components with variance 1, a standard Laplace distribution and a Student-T distribution with 3 degrees of freedom, we compare CR-AIS with four baselines in \Cref{tab:highdim}: Adaptive AIS with CESS decrease ratio of 0.6 (Ada. 0.6C), heuristic AIS with linear (Lin.), exponential (Exp.) and sigmoidal (Sigm.) schedules and Monte Carlo Diffusion (MCD) sampler \citep{doucet2022score} where the mean and diagonal variance of $q_0$, the schedule and the transitions are trained for 100 epochs maximizing the evidence lower bound. As CR-AIS and Adaptive AIS generate sequences of varying lengths, for better comparability, we also report a version of the algorithms with interpolated schedules indicated by asterisk ($^*$). See \Cref{sec:app:gm} for implementation details and further comparison with other adaptive baselines.

\begin{table*}[h]
    \caption{Absolute $\log Z_\pi$ estimation error for (Top) $d=128$ and (Bottom) $d=512$ dimensional distributions with $M$ close to $64$ (schedules with $^*$ use a shorter sequence for tuning and interpolate the result to $M=64$). Results are cross validated over different values of $\alpha$. Smallest error is in bold.}\label{tab:highdim}
    \vskip 0.15in
    \begin{center}
    \begin{small}
    \begin{sc}
    \sisetup{table-align-uncertainty=true, separate-uncertainty=true,}
    \renewrobustcmd{\bfseries}{\fontseries{b}\selectfont}
    \renewrobustcmd{\boldmath}{}
    \setlength\tabcolsep{2pt}
    \begin{tabular}{|c|cc|cc|cc|cc|}
    \hline
        \multirow{2}{*}{$\beta$} & \multicolumn{2}{c|}{Normal $\mathcal N(0, 0.01I)$} & \multicolumn{2}{c|}{Mixture} & \multicolumn{2}{c|}{Laplace} & \multicolumn{2}{c|}{Student-T} \\
            & Est. err. & Comput. & Est. err. & Comput. & Est. err. & Comput. & Est. err. & Comput. \\
    \hline
Lin. & 991.90 $\pm$ {\fontsize{6}{7.2}\selectfont 69.87} & 64.0 & 230.65 $\pm$ {\fontsize{6}{7.2}\selectfont 6.16} & 64.0 & 0.36 $\pm$ {\fontsize{6}{7.2}\selectfont 0.22} & 64.0 & 1.37 $\pm$ {\fontsize{6}{7.2}\selectfont 0.19} & 64.0\\
Sigm. & 907.57 $\pm$ {\fontsize{6}{7.2}\selectfont 24.04} & 64.0 & 270.55 $\pm$ {\fontsize{6}{7.2}\selectfont 6.42} & 64.0 & 0.07 $\pm$ {\fontsize{6}{7.2}\selectfont 1.14} & 64.0 & {1.46} $\pm$ {\fontsize{6}{7.2}\selectfont 0.14} & 64.0\\
Exp. & 780.58 $\pm$ {\fontsize{6}{7.2}\selectfont 40.02} & 64.0 & 507.43 $\pm$ {\fontsize{6}{7.2}\selectfont 15.04} & 64.0 & 1.21 $\pm$ {\fontsize{6}{7.2}\selectfont 0.51} & 64.0 & 2.04 $\pm$ {\fontsize{6}{7.2}\selectfont 0.31} & 64.0\\
Ada. 0.6c & 853.39 $\pm$ {\fontsize{6}{7.2}\selectfont 36.87} & 424.1 & 268.89 $\pm$ {\fontsize{6}{7.2}\selectfont 21.24} & 316.6 & 0.45 $\pm$ {\fontsize{6}{7.2}\selectfont 0.75} & 297.5 & 1.37 $\pm$ {\fontsize{6}{7.2}\selectfont 0.24} & 296.6\\
Ada. 0.6c$^*$ & 856.33 $\pm$ {\fontsize{6}{7.2}\selectfont 29.96} & 125.8 & 250.06 $\pm$ {\fontsize{6}{7.2}\selectfont 6.14} & 99.6  & 0.37 $\pm$ {\fontsize{6}{7.2}\selectfont 0.90} & 73.4 & 1.66 $\pm$ {\fontsize{6}{7.2}\selectfont 0.15} & 73.8\\
MCD & \bfseries{114.18} $\pm$ {\fontsize{6}{7.2}\selectfont 15.83} & 12800.0 & 680.60 $\pm$ {\fontsize{6}{7.2}\selectfont 13.68} & 12800.0 & 1.01 $\pm$ {\fontsize{6}{7.2}\selectfont 1.33} & 12800.0 & 1.99 $\pm$ {\fontsize{6}{7.2}\selectfont 1.26} & 12800.0\\
\hline
CR-AIS (Ours) & 788.67 $\pm$ {\fontsize{6}{7.2}\selectfont 36.25} & 78.0 & 308.82 $\pm$ {\fontsize{6}{7.2}\selectfont 6.81} & 112.0 & 0.68 $\pm$ {\fontsize{6}{7.2}\selectfont 0.31} & 117.2 & \bfseries{0.69} $\pm$ {\fontsize{6}{7.2}\selectfont 0.27} & 98.4\\
CR-AIS (Ours)$^*$ & 807.77 $\pm$ {\fontsize{6}{7.2}\selectfont 4.36} & 88.2 & \bfseries{228.54} $\pm$ {\fontsize{6}{7.2}\selectfont 0.00} & 72.0 & \bfseries{0.01} $\pm$ {\fontsize{6}{7.2}\selectfont 0.00} & 72.0 & {1.53} $\pm$ {\fontsize{6}{7.2}\selectfont 0.00} & 73.0\\

    \hline
    \hline
Lin. & 5279.55 $\pm$ {\fontsize{6}{7.2}\selectfont 79.56} & 64.0 & 1087.38 $\pm$ {\fontsize{6}{7.2}\selectfont 22.68} & 64.0 & 9.03 $\pm$ {\fontsize{6}{7.2}\selectfont 0.77} & 64.0 & \bfseries{8.16} $\pm$ {\fontsize{6}{7.2}\selectfont 0.45} & 64.0\\
Sigm. & 4757.13 $\pm$ {\fontsize{6}{7.2}\selectfont 59.34} & 64.0 & \bfseries{1022.60} $\pm$ {\fontsize{6}{7.2}\selectfont 23.55} & 64.0 & \bfseries{6.95} $\pm$ {\fontsize{6}{7.2}\selectfont 2.48} & 64.0 & 8.71 $\pm$ {\fontsize{6}{7.2}\selectfont 1.24} & 64.0\\
Exp. & {4435.05} $\pm$ {\fontsize{6}{7.2}\selectfont 110.36} & 64.0 & 1683.21 $\pm$ {\fontsize{6}{7.2}\selectfont 13.73} & 64.0 & 13.17 $\pm$ {\fontsize{6}{7.2}\selectfont 0.40} & 64.0 & 13.98 $\pm$ {\fontsize{6}{7.2}\selectfont 1.08} & 64.0\\
Ada. 0.6c & 5181.95 $\pm$ {\fontsize{6}{7.2}\selectfont 232.09} & 503.8 & 1147.02 $\pm$ {\fontsize{6}{7.2}\selectfont 142.66} & 367.2 & 9.12 $\pm$ {\fontsize{6}{7.2}\selectfont 3.01} & 413.8 & 9.72 $\pm$ {\fontsize{6}{7.2}\selectfont 2.41} & 247.7\\
Ada. 0.6c$^*$ & {4423.42} $\pm$ {\fontsize{6}{7.2}\selectfont 204.04} & 117.4  & 1272.70 $\pm$ {\fontsize{6}{7.2}\selectfont 158.61} & 85.2 & 8.27 $\pm$ {\fontsize{6}{7.2}\selectfont 1.33} & 94.8 & 9.30 $\pm$ {\fontsize{6}{7.2}\selectfont 0.56} & 81.6\\
MCD & \bfseries{707.85} $\pm$ {\fontsize{6}{7.2}\selectfont 110.02} & 12800.0 & 2173.25 $\pm$ {\fontsize{6}{7.2}\selectfont 43.78} & 12800.0 & 16.94 $\pm$ {\fontsize{6}{7.2}\selectfont 2.51} & 12800.0 & 20.95 $\pm$ {\fontsize{6}{7.2}\selectfont 0.93} & 12800.0\\
\hline
CR-AIS (Ours) & {4413.95} $\pm$ {\fontsize{6}{7.2}\selectfont 95.75} & 129.6 & 1200.12 $\pm$ {\fontsize{6}{7.2}\selectfont 27.12} & 140.0 & 8.75 $\pm$ {\fontsize{6}{7.2}\selectfont 1.84} & 110.8 & 8.40 $\pm$ {\fontsize{6}{7.2}\selectfont 1.28} & 98.4\\
CR-AIS (Ours)$^*$ & 4546.65 $\pm$ {\fontsize{6}{7.2}\selectfont 57.80} & 128.8 & {1069.32} $\pm$ {\fontsize{6}{7.2}\selectfont 1.96} & 75.0 & {7.93} $\pm$ {\fontsize{6}{7.2}\selectfont 0.74} & 78.6 & {8.98} $\pm$ {\fontsize{6}{7.2}\selectfont 0.00} & 73.0\\
\hline

    \end{tabular}
    \end{sc}
    \end{small}
    \end{center}
    \vskip -0.1in
\end{table*}

The average computation complexity of the sampling algorithms are reported in terms of the number of times $\log \tilde \pi$ is evaluated during tuning and testing. This value is proportional to the number of times $\log \tilde \pi$ or its gradient are evaluated which are generally the expensive part of the sampling. In Adaptive AIS this corresponds to the number of iterations in the search process of every update to the schedule for every step $i$ during tuning and the final number of discretization steps $M$ for the estimation phase. A parallel measure of complexity counts one schedule update per iteration during tuning in CR-AIS and one schedule update for each pass through the sequence for each annealing step in MCD during training. 

MCD has a clear advantage for the Normal target as the parameters of its initial distribution are trained to match the target. However, it's performance drops drastically on the other distributions with insufficient training and it is difficult to justify its expensive training overhead without amortization.

As expected, the performance of the heuristic schedules depend on the target. The exponential schedule is superior for the Normal target in both $d=128$ and $d=512$. The linear schedule and the sigmodal schedule are more accurate for the Student-T and the Laplace distributions, respectively, while their ranking changes with $d$ on the Gaussian mixture. Cross validating the heuristic would reduce the efficiency by three folds. On the other hand, at least one of the CR-AIS variations is able to beat the linear, the exponential, and the sigmoidal schedules in 7, 7 and 6 out of the 8 targets with an average overhead of \%40 due to tuning with interpolated schedule and \%70 without it.

CR-AIS is able to improve over Adaptive AIS while having a higher efficiency. In particular, computation complexity of non-interpolated Adaptive AIS is about $3.5\times$ more than CR-AIS on average, while CR-AIS estimations are more accurate for 5 of the 8 experiments. As the constant rate schedule preserves its form with different $\delta$ scales (see e.g. \Cref{sec:app:bayes}) CR-AIS can exploit this property leading to a better performance in comparison to interpolated Adaptive AIS in all 8 of the distributions.

\subsection{Bayesian Logistic Regression}\label{sec:experiment4}

In this section, we compare the computation efficiency of CR-AIS to heuristic and Adaptive AIS by evaluating the log marginal likelihood of two Bayesian models. We use two UCI datasets, Pima Indians diabetes dataset ($N=768$ and $d=8$) and Sonar dataset ($N=207$ and $d=60$) with binary labels and a setup similar to \cite{chopin2020introduction} with AIS. We consider a Bayesian logistic regression model with normal prior $p(z) = \mathcal N(0, 5I)$ and likelihood $p(\mathcal D|z) = \prod_n p(y_n|x_n, z)$ for $p(y_n | x_n, z) = \text{Bern}(\sigma(x_n^Tz))$. We use CR-AIS to estimate the log marginal likelihood $\log Z = p(\mathcal D) = \log \int p(z)p(\mathcal D|z)dz$ corresponding to the normalization factor of the posterior distribution of the parameters $\pi(z) = p(z|\mathcal D) \propto p(z)p(\mathcal D|z)$. 

For computation complexity we use a similar measure as described in \Cref{sec:experiment3} and plot the average of estimated $\log Z_\pi$ vs the computation complexity for Pima and Sonar datasets in \Cref{fig:pima-sonar}. The estimation of all samplers converges exponentially as the computation budget increases, while CR-AIS has tighter lower bound estimator in comparison to other samplers, especially when computation budget is limited and it roughly requires $\times 4$ fewer $\tilde \pi$ evaluations for similar performance as Adaptive AIS.  
\begin{figure}[h]
\vskip 0.2in
    \begin{center}
    \centerline{\includegraphics[width=0.45\textwidth]{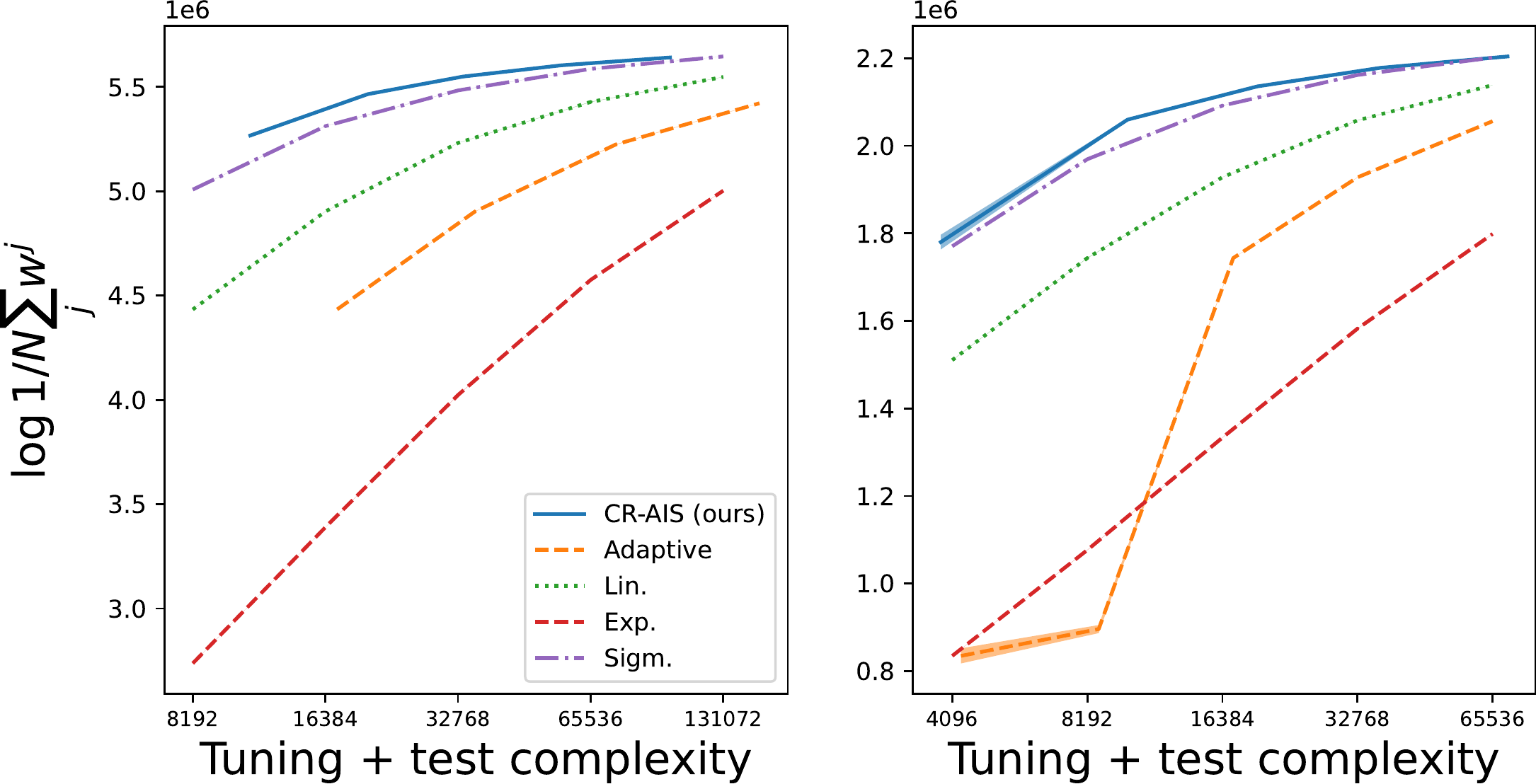}}
        \caption{Log marginal likelihood estimates of Bayesian logistic regression model vs computation complexity for CR-AIS, Adaptive AIS with ESS decrease rate of 0.5 on Pima (Left) and Sonar (Right) datasets.}
    \label{fig:pima-sonar}
    \end{center}
\vskip -0.2in
\end{figure}

\subsection{Latent Variable Model}

We estimate the log marginal likelihood of a Variational AutoEncoder trained on binarized MNIST dataset \citep{salakhutdinov2008quantitative} with a conditional Bernoulli likelihood model. We consider an architecture similar to the one used in \citep{burda2015importance} where the latent variable has $d=50$ dimensions and both the decoder and the encoder are neural networks with 3 fully-connected layers and we use the same set of hyperparameters as in \cref{sec:experiment3}. The estimated lower bound of log marginal likelihood are presented in \cref{tab:vae} for $M$ close to 64 and $M$ close to 512.

With $M$ close to 64, the samplers are far off from the estimation of the variational encoder which is -95.80 nats. The exponential schedule is more accurate than the rest of the samplers. CR-AIS gives a close estimate to it by doubling the computation time and requiring $3\times$ fewer resources in comparison to Adaptive AIS. For the longer sequences with $M$ close to 512, the performance of samplers is harder to distinguish and the gain of tuning diminishes.

\begin{table}[h]
    \caption{Estimated VAE log marginal likelihood with $M$ close to $64$ (Top) and $M$ close to 512 (Bottom). Results are cross validated over $\alpha$. Higher is better.}\label{tab:vae}
    \vskip 0.15in
    \begin{center}
    \begin{small}
    \begin{sc}
    \sisetup{table-align-uncertainty=true, separate-uncertainty=true,}
    \renewrobustcmd{\bfseries}{\fontseries{b}\selectfont}
    \renewrobustcmd{\boldmath}{}
    \setlength\tabcolsep{2pt}
    \begin{tabular}{|c|ccc|}
    \hline
        $\beta$ & Est. err. & $M$ & Comput. \\
    \hline
Lin. & -141.76 $\pm$ 0.07 & 64.0 & 64.0 \\
Sigm. & -130.65 $\pm$ 0.14 & 64.0 & 64.0 \\
Exp. & \bfseries{-121.04} $\pm$ 0.07 & 64.0 & 64.0 \\
Ada. c0.7 & -123.04 $\pm$ 3.66 & 81.6 & 395.6 \\
\hline
CR-AIS (Ours) & -124.63 $\pm$ 0.20 & 63.4 & 126.8 \\
    \hline
    \hline
Lin. & -106.44 $\pm$ 0.04 & 512.0 & 512.0 \\
Sigm. & -104.01 $\pm$ 0.07 & 512.0 & 512.0 \\
Exp. & \bfseries{-102.70} $\pm$ 0.06 & 512.0 & 512.0 \\
Ada. c0.7 & -103.26 $\pm$ 0.20 & 609.7 & 2560.0 \\
\hline
CR-AIS (Ours) & -104.24 $\pm$ 0.11 & 428.6 & 857.2 \\
\hline
    \end{tabular}
    \end{sc}
    \end{small}
    \end{center}
    \vskip -0.1in
\end{table}

\section{Related Work and Discussion}

Our work is similar to the first order optimization methods used to learn generative probability models. Optimizing functional of probability measures has been studied for decades. Functional gradients are computed in the space of probability distributions endowed with Hilbert structure \citep{liu2016stein, liu2017stein, dai2016provable, dai2018learning} or  Wasserstein structure \citep{frogner2020approximate, lin2021wasserstein}. 
Optimization of deep generative models is generalized as linearization of functional gradients  \citep{chu2019probability}. However, their application in bridging distributions has not received sufficient attention. These works focus on developing a Wasserstein gradient flow or particle flow to optimally reduce the objective. By contrast our work focuses on optimization of the intermediary unnormalized densities which are required for annealing in particle methods such as Annealed Stein Variational Gradient Descent \citep{d2021annealed}, Parallel Tempering \citep{earl2005parallel}, or Sequential Monte Carlo \citep{del2006sequential, naesseth2018variational}.

Traditionally, annealing schedule was tuned with ESS or its variants \citep{jasra2011inference, johansen2015towards, elvira2018rethinking}. A variational optimization of the annealing schedule was proposed in \citep{kiwaki2015variational} with fixed point iteration algorithm to minimize the asymptotic estimation bias/variance. In comparison, CR-AIS uses an analytically interpretable schedule and is able to replace the expensive numerical search or optimization loop with a simple Monte Carlo estimation.

Alternatively, many recent works emerge on combination of AIS and filtering with variational inference \citep{naesseth2018variational, maddison2017filtering, arbel2021annealed, thin2021monte} or score based generative models \cite{doucet2022score,doucet2022annealed} to achieve more complex transition kernel and priors and improve posterior approximation. In this line of work, the annealing schedule of a predetermined density path is treated as another set of parameters to optimize using back propagation. This approach results in better amortization for training deep latent variable models and higher log marginal likelihood. End-to-end optimization of the distribution path in AIS has proven to be a more challenging task \cite{zhang2021differentiable} and is not considered to be effective \cite{thin2021monte, geffner2021mcmc, goshtasbpour2023optimization}. Instead, we focus on a greedy optimization approach in terms of the marginal particle distribution divergence with the intended target distribution and we are able to provide a long missing understanding of the reasons underlying the popularity of heuristic annealing paths and demonstrate their limitations due to the greedy nature of the optimization.


\section{Conclusion}

In this work, we study the connection between the geometric density path in AIS and the functional derivative of inverse KL divergence of marginal particle distribution and the target. We prove that the geometric mean path is the solution to an ODE corresponding to the steepest descent direction of this objective. The analysis can be extended to $f$-divergences and the ODE has a closed form solution for $\alpha$-divergences in the form of power mean paths \citep{brekelmans2020annealed}. We derived constant rate schedule and designed an algorithm that achieves comparable results to the traditional adaptive AIS while avoiding the time consuming search procedure for tuning.

While our theory is motivated by reduction of the immediate bias of the log marginal likelihood estimator, the geometric path is not optimal with respect to the overall sampler bias as it doesn't use information from the possible future steps in the updates. Similarly, power mean path does not translate to an optimal end-to-end statistic of AIS importance weights. However, our optimization method is similar to a continuous time version of Adaptive AIS where instead of searching for the next discretization step size, we optimize the succeeding annealing densities in the function space. As a consequence, we provide a better understanding of the performance of the geometric mean heuristic and demonstrate the underlying reason for its suboptimality. 
We hope our work helps future research to develop alternative annealing paths with better end-to-end statistics.

\section*{Acknowledgements} This work is carried out by Shirin Goshtasbpour while supported by the funding from the European Union's Horizon 2020 research and innovation program under the Marie Sklodowska-Curie grant agreement No 813999 for this project.

\bibliography{references}

\begin{thebibliography}{41}
\providecommand{\natexlab}[1]{#1}
\providecommand{\url}[1]{\texttt{#1}}
\expandafter\ifx\csname urlstyle\endcsname\relax
  \providecommand{\doi}[1]{doi: #1}\else
  \providecommand{\doi}{doi: \begingroup \urlstyle{rm}\Url}\fi

\bibitem[Arbel et~al.(2021)Arbel, Matthews, and Doucet]{arbel2021annealed}
Arbel, M., Matthews, A., and Doucet, A.
\newblock Annealed flow transport monte carlo.
\newblock In \emph{International Conference on Machine Learning}, pp.\
  318--330. PMLR, 2021.

\bibitem[Brekelmans et~al.(2020)Brekelmans, Masrani, Bui, Wood, Galstyan,
  Steeg, and Nielsen]{brekelmans2020annealed}
Brekelmans, R., Masrani, V., Bui, T., Wood, F., Galstyan, A., Steeg, G.~V., and
  Nielsen, F.
\newblock Annealed {I}mportance {S}ampling with q-paths.
\newblock \emph{arXiv preprint arXiv:2012.07823}, 2020.

\bibitem[Burda et~al.(2016)Burda, Grosse, and
  Salakhutdinov]{burda2015importance}
Burda, Y., Grosse, R.~B., and Salakhutdinov, R.
\newblock Importance weighted autoencoders.
\newblock In Bengio, Y. and LeCun, Y. (eds.), \emph{4th International
  Conference on Learning Representations, {ICLR} 2016, San Juan, Puerto Rico,
  May 2-4, 2016, Conference Track Proceedings}, 2016.

\bibitem[Chopin et~al.(2020)Chopin, Papaspiliopoulos,
  et~al.]{chopin2020introduction}
Chopin, N., Papaspiliopoulos, O., et~al.
\newblock \emph{An introduction to {S}equential {M}onte {C}arlo}.
\newblock Springer, 2020.

\bibitem[Chu et~al.(2019)Chu, Blanchet, and Glynn]{chu2019probability}
Chu, C., Blanchet, J., and Glynn, P.
\newblock Probability functional descent: A unifying perspective on {G}{A}{N}s,
  variational inference, and reinforcement learning.
\newblock In \emph{International Conference on Machine Learning}, pp.\
  1213--1222. PMLR, 2019.

\bibitem[Dai(2018)]{dai2018learning}
Dai, B.
\newblock \emph{Learning over functions, distributions and dynamics via
  stochastic optimization}.
\newblock PhD thesis, Georgia Institute of Technology, 2018.

\bibitem[Dai et~al.(2016)Dai, He, Dai, and Song]{dai2016provable}
Dai, B., He, N., Dai, H., and Song, L.
\newblock Provable {B}ayesian inference via particle mirror descent.
\newblock In \emph{Artificial Intelligence and Statistics}, pp.\  985--994.
  PMLR, 2016.

\bibitem[D'Angelo \& Fortuin(2021)D'Angelo and Fortuin]{d2021annealed}
D'Angelo, F. and Fortuin, V.
\newblock Annealed {S}tein variational gradient descent.
\newblock \emph{arXiv preprint arXiv:2101.09815}, 2021.

\bibitem[Del~Moral et~al.(2006)Del~Moral, Doucet, and Jasra]{del2006sequential}
Del~Moral, P., Doucet, A., and Jasra, A.
\newblock Sequential {M}onte {C}arlo samplers.
\newblock \emph{Journal of the Royal Statistical Society: Series B (Statistical
  Methodology)}, 68\penalty0 (3):\penalty0 411--436, 2006.

\bibitem[Domke \& Sheldon(2018)Domke and Sheldon]{domke2018importance}
Domke, J. and Sheldon, D.~R.
\newblock Importance weighting and variational inference.
\newblock \emph{Advances in neural information processing systems}, 31, 2018.

\bibitem[Doucet et~al.(2022{\natexlab{a}})Doucet, Grathwohl, Matthews, and
  Strathmann]{doucet2022score}
Doucet, A., Grathwohl, W., Matthews, A.~G., and Strathmann, H.
\newblock Score-based diffusion meets annealed importance sampling.
\newblock \emph{Advances in Neural Information Processing Systems},
  35:\penalty0 21482--21494, 2022{\natexlab{a}}.

\bibitem[Doucet et~al.(2022{\natexlab{b}})Doucet, Grathwohl, Matthews, and
  Strathmann]{doucet2022annealed}
Doucet, A., Grathwohl, W.~S., Matthews, A. G. d.~G., and Strathmann, H.
\newblock Annealed importance sampling meets score matching.
\newblock In \emph{ICLR Workshop on Deep Generative Models for Highly
  Structured Data}, 2022{\natexlab{b}}.

\bibitem[Earl \& Deem(2005)Earl and Deem]{earl2005parallel}
Earl, D.~J. and Deem, M.~W.
\newblock Parallel tempering: Theory, applications, and new perspectives.
\newblock \emph{Physical Chemistry Chemical Physics}, 7\penalty0 (23):\penalty0
  3910--3916, 2005.

\bibitem[Elvira et~al.(2018)Elvira, Martino, and Robert]{elvira2018rethinking}
Elvira, V., Martino, L., and Robert, C.~P.
\newblock Rethinking the effective sample size.
\newblock \emph{arXiv preprint arXiv:1809.04129}, 2018.

\bibitem[Frogner \& Poggio(2020)Frogner and Poggio]{frogner2020approximate}
Frogner, C. and Poggio, T.
\newblock Approximate inference with wasserstein gradient flows.
\newblock In \emph{International Conference on Artificial Intelligence and
  Statistics}, pp.\  2581--2590. PMLR, 2020.

\bibitem[Geffner \& Domke(2021)Geffner and Domke]{geffner2021mcmc}
Geffner, T. and Domke, J.
\newblock {MCMC} variational inference via uncorrected hamiltonian annealing.
\newblock \emph{Advances in Neural Information Processing Systems},
  34:\penalty0 639--651, 2021.

\bibitem[Gelman \& Meng(1998)Gelman and Meng]{gelman1998simulating}
Gelman, A. and Meng, X.-L.
\newblock Simulating normalizing constants: From importance sampling to bridge
  sampling to path sampling.
\newblock \emph{Statistical science}, pp.\  163--185, 1998.

\bibitem[Goshtasbpour \& Perez-Cruz(2023)Goshtasbpour and
  Perez-Cruz]{goshtasbpour2023optimization}
Goshtasbpour, S. and Perez-Cruz, F.
\newblock Optimization of annealed importance sampling hyperparameters.
\newblock In \emph{Machine Learning and Knowledge Discovery in Databases:
  European Conference, ECML PKDD 2022, Grenoble, France, September 19--23,
  2022, Proceedings, Part V}, pp.\  174--190. Springer, 2023.

\bibitem[Grosse et~al.(2013)Grosse, Maddison, and
  Salakhutdinov]{grosse2013annealing}
Grosse, R.~B., Maddison, C.~J., and Salakhutdinov, R.
\newblock Annealing between distributions by averaging moments.
\newblock In \emph{NIPS}, pp.\  2769--2777. Citeseer, 2013.

\bibitem[Grosse et~al.(2015)Grosse, Ghahramani, and
  Adams]{grosse2015sandwiching}
Grosse, R.~B., Ghahramani, Z., and Adams, R.~P.
\newblock Sandwiching the marginal likelihood using bidirectional {M}onte
  {C}arlo.
\newblock \emph{CoRR}, abs/1511.02543, 2015.
\newblock URL \url{http://arxiv.org/abs/1511.02543}.

\bibitem[Grosse et~al.(2016)Grosse, Ancha, and Roy]{grosse2016measuring}
Grosse, R.~B., Ancha, S., and Roy, D.~M.
\newblock Measuring the reliability of {MCMC} inference with bidirectional
  {M}onte {C}arlo.
\newblock \emph{Advances in Neural Information Processing Systems}, 29, 2016.

\bibitem[Jasra et~al.(2011)Jasra, Stephens, Doucet, and
  Tsagaris]{jasra2011inference}
Jasra, A., Stephens, D.~A., Doucet, A., and Tsagaris, T.
\newblock Inference for {L}{\'e}vy-driven stochastic volatility models via
  adaptive {S}equential {M}onte {C}arlo.
\newblock \emph{Scandinavian Journal of Statistics}, 38\penalty0 (1):\penalty0
  1--22, 2011.

\bibitem[Johansen et~al.(2015)Johansen, Aston, and Zhou]{johansen2015towards}
Johansen, A.~M., Aston, J.~A., and Zhou, Y.
\newblock Towards automatic model comparison: An adaptive {S}equential {M}onte
  {C}arlo approach.
\newblock 2015.

\bibitem[Kiwaki(2015)]{kiwaki2015variational}
Kiwaki, T.
\newblock Variational optimization of annealing schedules.
\newblock \emph{arXiv preprint arXiv:1502.05313}, 2015.

\bibitem[Kong(1992)]{kong1992note}
Kong, A.
\newblock A note on importance sampling using standardized weights.
\newblock \emph{University of Chicago, Dept. of Statistics, Tech. Rep}, 348,
  1992.

\bibitem[Lin et~al.(2021)Lin, Li, Osher, and Mont{\'u}far]{lin2021wasserstein}
Lin, A.~T., Li, W., Osher, S., and Mont{\'u}far, G.
\newblock Wasserstein proximal of {G}{A}{N}s.
\newblock In \emph{International Conference on Geometric Science of
  Information}, pp.\  524--533. Springer, 2021.

\bibitem[Liu(2017)]{liu2017stein}
Liu, Q.
\newblock Stein variational gradient descent as gradient flow.
\newblock \emph{Advances in neural information processing systems}, 30, 2017.

\bibitem[Liu \& Wang(2016)Liu and Wang]{liu2016stein}
Liu, Q. and Wang, D.
\newblock Stein variational gradient descent: A general purpose {B}ayesian
  inference algorithm.
\newblock \emph{Advances in neural information processing systems}, 29, 2016.

\bibitem[Maddison et~al.(2017)Maddison, Lawson, Tucker, Heess, Norouzi, Mnih,
  Doucet, and Teh]{maddison2017filtering}
Maddison, C.~J., Lawson, J., Tucker, G., Heess, N., Norouzi, M., Mnih, A.,
  Doucet, A., and Teh, Y.
\newblock Filtering variational objectives.
\newblock \emph{Advances in Neural Information Processing Systems}, 30, 2017.

\bibitem[Masrani et~al.(2019)Masrani, Le, and Wood]{masrani2019thermodynamic}
Masrani, V., Le, T.~A., and Wood, F.
\newblock The thermodynamic variational objective.
\newblock \emph{Advances in Neural Information Processing Systems}, 32, 2019.

\bibitem[Masrani et~al.(2021)Masrani, Brekelmans, Bui, Nielsen, Galstyan,
  Ver~Steeg, and Wood]{masrani2021q}
Masrani, V., Brekelmans, R., Bui, T., Nielsen, F., Galstyan, A., Ver~Steeg, G.,
  and Wood, F.
\newblock q-paths: Generalizing the geometric annealing path using power means.
\newblock In \emph{Uncertainty in Artificial Intelligence}, pp.\  1938--1947.
  PMLR, 2021.

\bibitem[Naesseth et~al.(2018)Naesseth, Linderman, Ranganath, and
  Blei]{naesseth2018variational}
Naesseth, C., Linderman, S., Ranganath, R., and Blei, D.
\newblock Variational {S}equential {M}onte {C}arlo.
\newblock In \emph{International conference on artificial intelligence and
  statistics}, pp.\  968--977. PMLR, 2018.

\bibitem[Neal(1996)]{neal1996sampling}
Neal, R.~M.
\newblock Sampling from multimodal distributions using tempered transitions.
\newblock \emph{Statistics and computing}, 6\penalty0 (4):\penalty0 353--366,
  1996.

\bibitem[Neal(2001)]{neal2001annealed}
Neal, R.~M.
\newblock {A}nnealed {I}mportance {S}ampling.
\newblock \emph{Statistics and computing}, 11\penalty0 (2):\penalty0 125--139,
  2001.

\bibitem[Ogata(1989)]{ogata1989monte}
Ogata, Y.
\newblock A {M}onte {C}arlo method for high dimensional integration.
\newblock \emph{Numerische Mathematik}, 55\penalty0 (2):\penalty0 137--157,
  1989.

\bibitem[Rezende et~al.(2014)Rezende, Mohamed, and
  Wierstra]{rezende2014stochastic}
Rezende, D.~J., Mohamed, S., and Wierstra, D.
\newblock Stochastic backpropagation and approximate inference in deep
  generative models.
\newblock In \emph{International conference on machine learning}, pp.\
  1278--1286. PMLR, 2014.

\bibitem[Salakhutdinov \& Murray(2008)Salakhutdinov and
  Murray]{salakhutdinov2008quantitative}
Salakhutdinov, R. and Murray, I.
\newblock On the quantitative analysis of deep belief networks.
\newblock In \emph{Proceedings of the 25th international conference on Machine
  learning}, pp.\  872--879, 2008.

\bibitem[Thin et~al.(2021)Thin, Kotelevskii, Doucet, Durmus, Moulines, and
  Panov]{thin2021monte}
Thin, A., Kotelevskii, N., Doucet, A., Durmus, A., Moulines, E., and Panov, M.
\newblock Monte {C}arlo variational auto-encoders.
\newblock In \emph{International Conference on Machine Learning}, pp.\
  10247--10257. PMLR, 2021.

\bibitem[Wu et~al.(2020)Wu, K{\"o}hler, and No{\'e}]{wu2020stochastic}
Wu, H., K{\"o}hler, J., and No{\'e}, F.
\newblock Stochastic normalizing flows.
\newblock \emph{arXiv preprint arXiv:2002.06707}, 2020.

\bibitem[Wu et~al.(2017)Wu, Burda, Salakhutdinov, and Grosse]{wu2017on}
Wu, Y., Burda, Y., Salakhutdinov, R., and Grosse, R.
\newblock On the quantitative analysis of decoder-based generative models.
\newblock In \emph{International Conference on Learning Representations}, 2017.

\bibitem[Zhang et~al.(2021)Zhang, Hsu, Li, Finn, and
  Grosse]{zhang2021differentiable}
Zhang, G., Hsu, K., Li, J., Finn, C., and Grosse, R.~B.
\newblock Differentiable annealed importance sampling and the perils of
  gradient noise.
\newblock \emph{Advances in Neural Information Processing Systems},
  34:\penalty0 19398--19410, 2021.

\end{thebibliography}
\bibliographystyle{icml2023}

\newpage
\appendix
\onecolumn
\section{Proofs}\label{sec:app:proofs}
Here, we provide the proofs of the lemmas and propositions in the paper.

\begin{lemma}
    Let $\phi_{t+\eps}(z) = \phi_t(z) + \eps \eta(z)$ for $\phi_t(z) = \log \tilde q_t(z)$ where $\tilde q_t(z)$ and $\tilde \pi(z)$ are positive unnormalized density functions. Then we have,
    \begin{align}\label{p:kl-derivative}
        \at{\frac{d}{d\eps}J_\textup{KL}[\phi_{t+\eps}]}{\eps=0} = &\textup{Cov}_{q_t}\left[\eta(z), \log \frac{\tilde q_t(z)}{\tilde \pi(z)}\right],
    \end{align}
    where $\textup{Cov}_q[\cdot,\cdot]$ is the covariance under distribution of $q$ and we use the definition of G\^{a}teaux differential for the derivative,
    \begin{align}
        \at{\frac{d}{d\eps}J_\textup{KL}[\phi_{t+\eps}]}{\eps=0} = &\lim_{\eps \to 0^+} \frac{J_\textup{KL}[\phi_t(z) + \eps \eta(z)] - J_\textup{KL}[\phi_t]}{\eps}. \nonumber
    \end{align}
\end{lemma}

\begin{proof}
    We take the derivative as follows,
    \begin{align}
    \at{\frac{d}{d\eps}J_\text{KL}[\phi_{t+\eps}]}{\eps=0} = &\frac{\int \tilde q_t(z) \eta(z) \log \frac{\tilde q_t(z)}{\tilde \pi(z)} dz}{\int \tilde q_t(z) dz}+ \frac{\int \tilde q_t(z) \eta(z) dz}{\int \tilde q_t(z) dz}\nonumber\\
    &- \frac{\int \tilde q_t(z) \log \frac{\tilde q_t(z)}{\tilde \pi(z)} dz \int \tilde q_t(z) \eta(z) dz}{\int \tilde q_t(z) dz\int \tilde q_t(z) dz}- \frac{\int \tilde q_t(z) \eta(z) dz}{\int \tilde q_t(z) dz}\nonumber\\
    = &\text{Cov}_{q_t}\left[\eta(z), \log \frac{\tilde q_t(z)}{\tilde \pi(z)}\right],\nonumber
\end{align}
concluding the proof.
\end{proof}

\begin{proposition}
    Assume the same conditions as in \Cref{lem:kl-derivative}. Additionally, consider the set of smooth perturbation directions with bounded variance \begin{align}
        \mathcal M_{q_t,\pi}\coloneqq \{\eta\in \mathcal C^1: \textup{Var}_{q_t}[\eta(z)]\leq c_{q_t,\pi}^\text{KL}\},\nonumber
    \end{align}
    for $B \geq 0$ and $c_{q_t,\pi}^\text{KL} = B/\textup{Var}_{q_t}[\log (\tilde \pi(z)/\tilde q_t(z))]$.
    Then the steepest descent direction that minimizes the derivative in \Cref{eq:kl-derivative} in $\mathcal M_{q_t,\pi}$ is
    \begin{align}\label{p:etaopt-kl}
        \eta^*_{q_t,\pi}(z) = \frac {c_{q_t,\pi}^\text{KL}}{\sqrt B} \log\frac{\tilde\pi(z)}{\tilde q_t(z)} + b,
    \end{align}
    for arbitrary $b\in\R$. A solution to the Ordinary Differential Equation (ODE) $\frac{d}{dt}\phi_t(z) = \eta^*_{q_t, \pi}(z)$ with initial condition $\phi_0(z) = \log \tilde q_0(z)$ is the scaled geometric mean path and results in constant rate decrease in the inverse KL divergence in \Cref{eq:kl}.
\end{proposition}
\begin{proof}
    It is straight forward to derive the equation for the steepest descent direction using Cauchy-Schwarz inequality. The solution to the ODE is given by
    \begin{align}\label{p:geom}
        \phi_t(z) = &\beta(t) \log \tilde q_0(z) + (1 - \beta(t))\log \tilde \pi(z) \nonumber\\
        &+ \beta(t)\int_0^t \frac{b(r)}{\beta(r)}dr,
    \end{align}
    which is equivalent to $\log \tilde q^\text{geom}_{1 - \beta(t)}(z)$ plus a $z$-independent scale for $t\geq 0$ and for $\beta(t)$ set as 
    \begin{align}\label{p:betageom}
        \beta^\text{KL}(t) \coloneqq e^{-\int_0^tc_{q_r,\pi}^\text{KL}dr/\sqrt B},
    \end{align}
    it leads in constant derivative value form \Cref{eq:kl-derivative}.
\end{proof}

\begin{lemma}
    Let $\phi_{t+\eps(z)} = \phi_t(z) + \eps \eta(z)$ for $\phi_t(z) = \log \tilde q_t(z)$ where $\tilde q_t(z)$ and $\tilde\pi(z)$ are positive unnormalized density functions and let $f:\R\to\R$ be convex and differentiable. Then we have,
    \begin{align}
        \at{\frac{d}{d\eps}J_f[\phi_{t+\eps}]}{\eps=0} = &\textup{Cov}_{q_t}\left[\eta(z), -g(u_t(z))\right],\nonumber
    \end{align}
    where $g(u) = u\dot f(u)-f(u) $ and $\dot f(u) = df(u)/du$. Moreover, consider the set of smooth perturbation directions with bounded variance $\mathcal M^f_{q_t,\pi}\coloneqq \left\{\eta\in \mathcal C^1: \textup{Var}_{q_t}[\eta(z)]\leq c_{q_t,\pi}^f\right\}$ for $B \geq 0$ and
    \begin{align}
        c_{q_t,\pi}^f = B/\textup{Var}_{q_t}\left[g(u_t(z))\right].\nonumber
    \end{align}
    Then the steepest descent direction that minimizes this derivative in $\mathcal M_{q_t,\pi}^f$ is
    \begin{align}
        \eta^*_{q_t,\pi}(z) = \frac{c_{q_t,\pi}^f}{\sqrt B}g(u_t(z)) + b,\nonumber
    \end{align}
    for arbitrary $b\in\R$.
\end{lemma}
\begin{proof}
    With the perturbed negative energy
    \begin{align}
        \phi_{t+\eps}(z) = \log q_{t+\eps}(z) = \phi_t(z) + \eps\eta(z)\nonumber
    \end{align}
    in direction $\eta$, $u_t(z)$ is updated to $u_t(z)e^{-\eps\eta(z)}\E_{q_t}[e^{\eps\eta(z)}]$. Consequently,
\begin{align}
    \at{\frac{d}{d\eps}J_f[\phi_{t+\eps}]}{\eps=0} = &\frac{\int \tilde q_t(z)\eta(z) f\left(u_t(z)\right)dz}{\int \tilde q_t(z) dz} + \frac{\int \tilde q_t(z) \frac{d}{du}f\left(u_t(z)\right)u_t(z)\E_{q_t}[\eta(z)] dz}{\int \tilde q_t(z) dz}\nonumber\\
    &- \frac{\int \tilde q_t(z) \frac{d}{du}f\left(u_t(z)\right)u_t(z)\eta(z) dz}{\int \tilde q_t(z) dz} - \frac{\int \tilde q_t(z) f\left(u_t(z)\right) dz}{\int \tilde q_t(z) dz}.\frac{\int \tilde q_t(z) \eta(z) dz}{\int \tilde q_t(z) dz}\nonumber\\
    = &\text{Cov}_{q_t}\left[\eta(z), -g(u_t(z))\right].\nonumber
\end{align}
    The rest of the proof follows due to Cauchy-Schwarz inequality.
\end{proof}

\begin{proposition}
    Assume the same conditions as in \Cref{lem:f-derivative} and $f(u) = (u^\alpha - 1 - \alpha(u-1))/\alpha(\alpha-1)$ for $\alpha \centernot\in \{0,1\}$ or $f(u) = u\log u$ for $\alpha=1$. Then $\alpha$-power mean path is the solution to the ODE $\frac{d}{dt}\phi_t(z) = \eta^*_{q_t, \pi}(z)$ with initial condition $\phi_0(z) = \log \tilde q_0(z)$ and with a particular schedule it results in constant rate decrease in $f$-divergence.
\end{proposition}
\begin{proof}
    We prove the proposition for $\alpha\centernot\in\{0,1\}$ in the following which can be easly extended to the case with $\alpha=1$.
    Using the derivative of $f$ we have $g(u) = (u^\alpha - 1)/\alpha$, therefore,
    \begin{align}
        \eta^*_{q_t,\pi}(z) = \frac{c_{q_t,\pi}^f}{\alpha\sqrt B}(u_t^\alpha - 1) + b(t).\nonumber
    \end{align}
    In the ODE
    \begin{align}
        \frac{d}{dt}\log \tilde q(z) &= \eta_{q_t,\pi}^*(z) + b(t)  \nonumber\\
        &= \frac{c_{q_t,\pi}^f}{\alpha\sqrt B}\left( \left(\frac{\tilde \pi(z)}{\tilde q_t(z)}\frac{Z_{q_t}}{Z_\pi}\right)^\alpha - 1\right) + b(t).
    \end{align} 
    We use the change of variable $\tilde s_t(z) = \tilde q_t(z)^\alpha$ and rewrite the ODE as
    \begin{align}
        \frac{d}{dt}\tilde s_t(z) + \left(\frac{c_{q_t,\pi}^f}{\sqrt B} - \alpha b(t)\right)\tilde s_t(z) = \frac{c_{q_t,\pi}^f Z_{q_t}^\alpha\tilde \pi(z)^\alpha}{\sqrt B Z_\pi^\alpha},\nonumber
    \end{align}
    which has a solution of the form
    \begin{align}
        s_t(z) = \beta(t)\left[\tilde \pi(z)^\alpha\int_0^t \frac{c_{q_r,\pi}^f Z_{q_r}^\alpha}{\sqrt B Z_\pi^\alpha \beta(r)}dr + \tilde q_0(z)^\alpha\right]\nonumber
    \end{align}
    for 
    \begin{align}
        \beta(t) = e^{\int_0^t (\alpha b(r)- c_{q_r,\pi}^f/\sqrt B) dr}.\nonumber
    \end{align}
    Therefore, without loss of generality, we choose 
    \begin{align}
        b(t) = -c_{q_t,\pi}^f((Z_{q_t}/Z_\pi)^\alpha - 1)/\alpha\sqrt B\nonumber
    \end{align}
    and get the annealing sequence $\tilde q^{\alpha\text{-pow}}_{\tau(t)}(z)$ for the schedule $\tau(t) = 1 - \beta(t)$ and with $\beta(t)$ set to 
    \begin{align}\label{p:beta}
        \beta^\alpha(t) \coloneqq e^{-\int_0^t (c_{q_r,\pi}^fZ_{q_r}^\alpha/\sqrt BZ_\pi^\alpha)dr},
    \end{align}
    The $\alpha$-divergence decreases with a steady rate along the steepest descent direction.
\end{proof}

\section{2D Distributions Implementation Details}\label{sec:app:2d}

Here we give the implementation details of our experiments on 2d benchmark dataset. 
We initialize CR-AIS with a standard normal distribution for $q_0$ which is plotted in the top right corner of \Cref{fig:particles} for scale and we use the same value of $\delta=1/32$ for all the targets. Each AIS transition is a single Hamiltonian Monte Carlo (HMC) step with step size 0.5 and $N=1024$ particles are used to approximate the constant rate schedule. We abort sampling when the empirical variance $\widehat{\text{Var}}_{q_{i\delta}}[g(u)]$ is below $0.001$. We use $\alpha=0$ for the results in \Cref{sec:experiment1}. 

In \Cref{sec:experiment2}, for Adaptive AIS, we tune the schedule on each iteration using binary search with constrained maximum step size set to $1/128$ avoid large steps. This value is chosen to give the search algorithm sufficient flexibility to effect the schedule while larger values would result in big steps and premature annealing.

\section{Accuracy of the Approximations in CR-AIS}\label{sec:app:f}

To assess the accuracy of the approximations used in the CR-AIS algorithm, we depict the Monte Carlo estimation of the objective during a run of the algorithm for the same setup as \Cref{sec:experiment2} in \Cref{fig:f}. The curves indicate an almost static decrease of the inverse KL divergence and the efficiency of the approximations in CR-AIS with combination of self-normalized normalization factor ratio estimation and Riemann sum. 

\begin{figure}[h]
    \begin{center}
    \centerline{\includegraphics[width=0.3\textwidth]{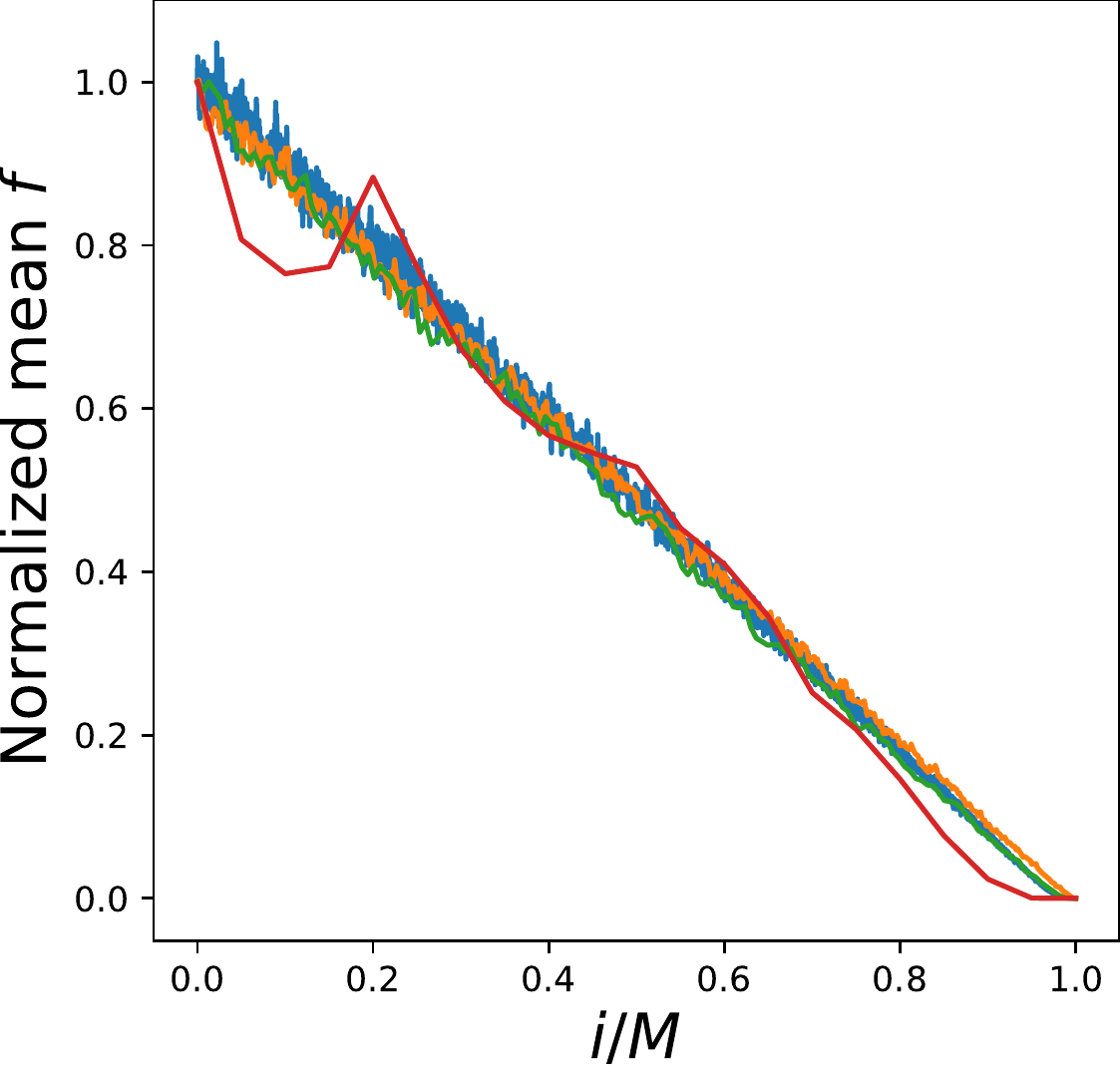}}
        \caption{Mean objective vs normalized AIS iterations in four distributions used in \Cref{sec:experiment2}, with geometric mean path and $\delta = 1/32$. Mean $f$ is normalized by the initial divergence to fit in the same plot. Best seen in color.}
    \label{fig:f}
    \end{center}
    \vskip -0.2in
\end{figure}

\section{Constant Rate Sequential Monte Carlo Experiments}\label{sec:app:smc}

In this section, we adapt Sequential Monte Carlo (SMC) algorithm using the constant rate schedule and compare the performance of sampling to Adaptive SMC.
SMC is a modification of AIS with occasional resampling steps to prevent weight degeneracy problem in sequential importance sampling algorithms. In practice, the resampling is performed periodically or adaptively whenever ESS falls below a threshold, the latter providing more sample diversity.

To adapt the schedule in SMC we used weighted estimates of log normalization factor ratio with the SMC importance weights (line 7-9). The rest of the algorithm is similar to \Cref{algo} with the additional branch for resampling after line 18 to duplicate the effective samples. We perform resampling adaptively when ESS is less than 0.9 of the last time resampling was performed.

We construct 4 randomly generated wide Gaussian mixtures uniformly sampling the mean of the components from $[-15, -15]$ and compare the synthesized particles to Adaptive SMC where the schedule is determined with binary search on each iteration to achieve constant rate decrease of $0.9$ in ESS. \Cref{fig:smc} shows the sampled particles from CR-SMC and Adaptive AIS starting from standard normal distribution with similar HMC transitions. Adapative SMC particles tend to group together in a few of the targets' modes, while with CR-SMC the final samples have higher diversity and reach all the target modes in three of the four distributions. This happens despite the fact that CR-SMC uses shorter annealing sequences in all of the distributions.

\begin{figure}[h]
    \centering
        \centerline{\includegraphics[width=1.\textwidth]{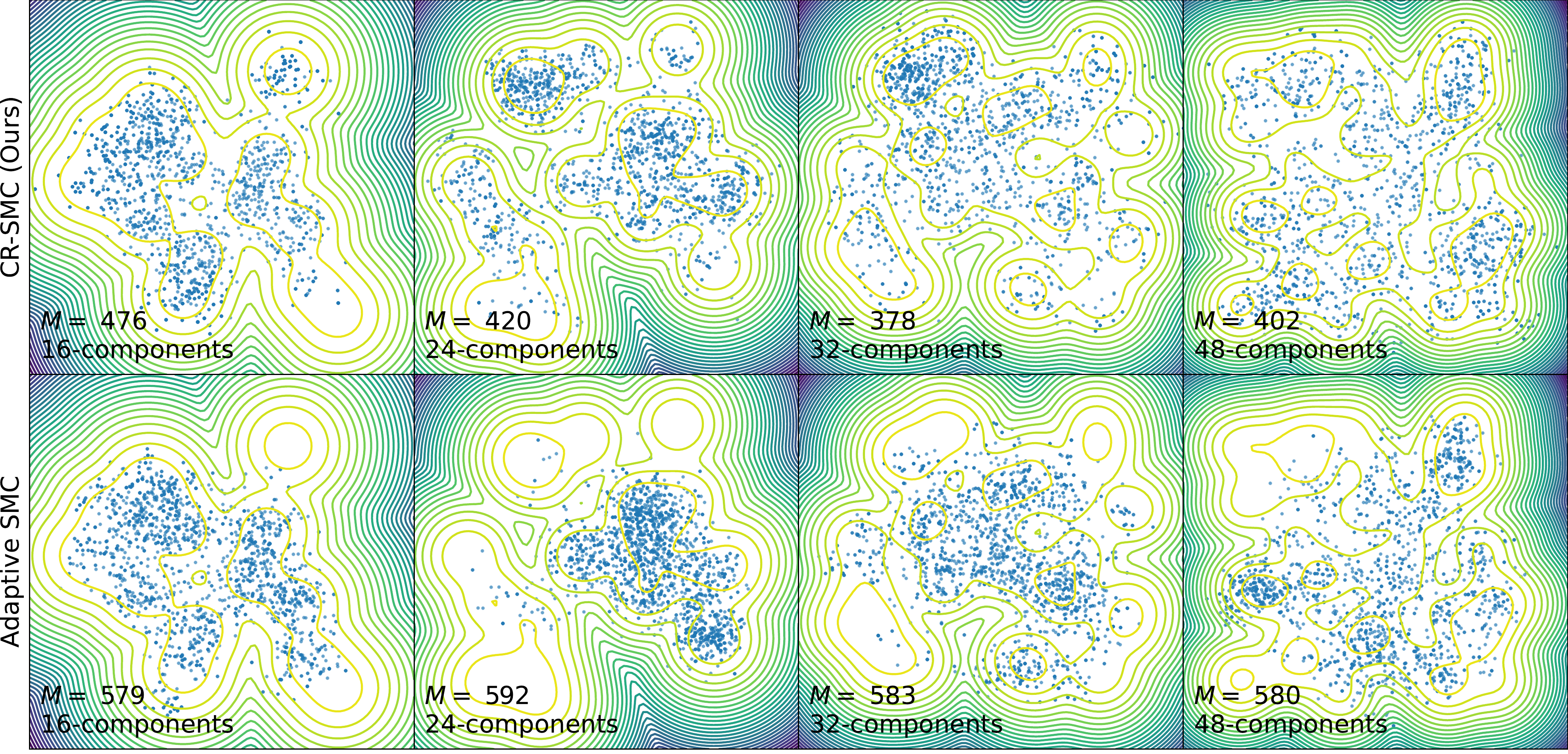}}
        \caption{Contour lines of four wide mixture distributions in [-20,20] in each column with samples (blue discs) from CR-SMC (top) and Adaptive SMC (bottom). geometric mean path with $\delta=1/128$ selected to have roughly the same number of iterations as Adaptive SMC which has worse coverage.}
        \label{fig:smc}
\end{figure}

In \Cref{tab:smc} we report the $\log Z_\pi$ estimations with both algorithms on the four targets.

\begin{table}[h]
    \caption{Estimated $\log Z$ of four gaussian mixtures with the same order as Figure~\ref{fig:smc} with 16,24,32 and 48 components using CR-SMC and Adaptive SMC with 0.9 CESS drop resampling and 0.9 ESS drop schedule adaptation. Lower estimation error is in bold.} \label{tab:smc}
    \begin{center}
    \begin{small}
    \begin{sc}
        \sisetup{
      table-align-uncertainty=true,
      separate-uncertainty=true,
    }
    \renewrobustcmd{\bfseries}{\fontseries{b}\selectfont}
    \renewrobustcmd{\boldmath}{}
    \setlength\tabcolsep{5pt}
    
    \begin{tabular}{|c|cccc|}
    \hline
         & Target 1 & Target 2 & Target 3 & Target 4 \\
         $\log Z$ & 5.997 &6.402 & 6.690 & 7.095\\
    \hline
         Ada 0.9 & 5.387 & \bfseries{6.285} & 5.891 & 6.808\\
         CR-SMC (Ours) & \bfseries{5.967} & 6.125 & \bfseries{6.634} & \bfseries{6.974}\\
    \hline
    \end{tabular}
    \end{sc}
    \end{small}
    \end{center}
\end{table}

\section{High dimensional Experiments}\label{sec:app:gm}

We use the same hyperparameters as in \Cref{sec:app:2d} for the log normalization factor estimation experiments with the following modifications. We set $N=4096$ and chose $\alpha$ from a grid in $\{-0.5, 0.0, ..., 2.\}$, $\delta $ from $\{1/256, ..., 4096\}$ for CR-AIS depending on the target such that sampling budget $M$ will be $M\leq64$ for interpolation and close to $M=64$ without interpolation.
To generate the results with interpolated schedules, we interpolate the tuned schedules with each $\delta$ to match the test-time computation of all samplers to $M=64$ and sample from the target distribution. 
We report the $\log 1/N \sum w(z_{0:M}^j)$ with smallest estimation error over $\alpha$ and $\delta$. In order to avoid the large steps in Adaptive AIS when ESS is underestimating the variance, we choose a maximum step size in the adaptive AIS from $\{1/512, ..., 1/4\}$ using cross validation in the same manner as CR-AIS. Each algorithm is run with 5 different seeds and the average absolute estimation error is reported.

Below we present more thorough comparison of CR-AIS with adaptive and heuristic AIS benchmarks. The details of the experiment are similar to the setup in Section{sec:experiment3} with additional variations of Adaptive AIS algorithm. Specifically, in Adaptive AIS, we tune the schedule on each iteration using binary search with constrained maximum step size set to avoid large steps and apply proposed steps which decrease Effective Sample Size (ESS) and Conditional ESS (CESS) \cite{johansen2015towards} constantly with ratio in $\{0.5, 0.6, 0.7, 0.8, 0.9\}$.

\begin{table*}[h]
\caption{Absolute $\log Z_\pi$ estimation error for 128-dimensional distributions: Normal $\mathcal N(0, 0.01I)$, Gaussian mixtures with 8 components of variance 1, standard Laplace and Student-T distribution with 3 degrees of freedom with $M$ close to 64. Results are cross validated over different values of $\alpha$. Smallest error is in bold.}\label{tab:logZ}
    \begin{center}
    \begin{small}
    \begin{sc}
    \sisetup{table-align-uncertainty=true, separate-uncertainty=true,}
    \renewrobustcmd{\bfseries}{\fontseries{b}\selectfont}
    \renewrobustcmd{\boldmath}{}
    \setlength\tabcolsep{2pt}
    \begin{tabular}{|c|cc|cc|cc|cc|}
    \hline
        \multirow{2}{*}{$\beta$} & \multicolumn{2}{c|}{Normal} & \multicolumn{2}{c|}{Mixture} & \multicolumn{2}{c|}{Laplace} & \multicolumn{2}{c|}{Student-T} \\
            & Est. err. & $M$ & Est. err. & $M$ & Est. err. & $M$ & Est. err. & $M$ \\
    \hline
Ada. 0.9c & 871.75 $\pm$ {\fontsize{6}{7.2}\selectfont 25.81} & 69.8 & 246.08 $\pm$ {\fontsize{6}{7.2}\selectfont 33.20} & 69.2 & 0.80 $\pm$ {\fontsize{6}{7.2}\selectfont 0.65} & 43.2 & 0.96 $\pm$ {\fontsize{6}{7.2}\selectfont 0.66} & 50.0\\
Ada. 0.6c & 853.39 $\pm$ {\fontsize{6}{7.2}\selectfont 36.87} & 61.8 & 268.89 $\pm$ {\fontsize{6}{7.2}\selectfont 21.24} & 61.4 & 0.45 $\pm$ {\fontsize{6}{7.2}\selectfont 0.75} & 39.6 & 1.37 $\pm$ {\fontsize{6}{7.2}\selectfont 0.24} & 43.2\\
Ada. 0.7c & 864.52 $\pm$ {\fontsize{6}{7.2}\selectfont 31.07} & 61.6 & 253.38 $\pm$ {\fontsize{6}{7.2}\selectfont 39.73} & 70.0 & 0.12 $\pm$ {\fontsize{6}{7.2}\selectfont 0.55} & 69.8 & 1.13 $\pm$ {\fontsize{6}{7.2}\selectfont 0.39} & 43.6\\
Ada. 0.8c & 879.29 $\pm$ {\fontsize{6}{7.2}\selectfont 22.60} & 63.2 & 252.01 $\pm$ {\fontsize{6}{7.2}\selectfont 33.43} & 69.0 & 0.74 $\pm$ {\fontsize{6}{7.2}\selectfont 0.97} & 40.8 & 0.96 $\pm$ {\fontsize{6}{7.2}\selectfont 1.75} & 11.4\\
Ada. 0.9 & 992.07 $\pm$ {\fontsize{6}{7.2}\selectfont 33.45} & 66.4 & 347.07 $\pm$ {\fontsize{6}{7.2}\selectfont 29.13} & 43.0 & \bfseries{0.04} $\pm$ {\fontsize{6}{7.2}\selectfont 0.77} & 44.8 & 1.00 $\pm$ {\fontsize{6}{7.2}\selectfont 1.25} & 42.0\\
Ada. 0.6 & 1001.14 $\pm$ {\fontsize{6}{7.2}\selectfont 72.16} & 65.6 & 384.80 $\pm$ {\fontsize{6}{7.2}\selectfont 24.05} & 36.8 & 0.72 $\pm$ {\fontsize{6}{7.2}\selectfont 0.50} & 37.2 & 0.92 $\pm$ {\fontsize{6}{7.2}\selectfont 0.83} & 37.2\\
Ada. 0.7 & 986.59 $\pm$ {\fontsize{6}{7.2}\selectfont 53.15} & 65.8 & 378.61 $\pm$ {\fontsize{6}{7.2}\selectfont 23.42} & 38.0 & 0.54 $\pm$ {\fontsize{6}{7.2}\selectfont 0.62} & 39.4 & 1.56 $\pm$ {\fontsize{6}{7.2}\selectfont 0.29} & 41.2\\
Ada. 0.8 & 987.17 $\pm$ {\fontsize{6}{7.2}\selectfont 29.88} & 66.8 & 357.53 $\pm$ {\fontsize{6}{7.2}\selectfont 22.14} & 40.6 & 0.15 $\pm$ {\fontsize{6}{7.2}\selectfont 0.70} & 42.0 & 1.66 $\pm$ {\fontsize{6}{7.2}\selectfont 0.14} & 39.8\\
Lin. & 991.90 $\pm$ {\fontsize{6}{7.2}\selectfont 69.87} & 64.0 & \bfseries{230.65} $\pm$ {\fontsize{6}{7.2}\selectfont 6.16} & 64.0 & 0.36 $\pm$ {\fontsize{6}{7.2}\selectfont 0.22} & 64.0 & 1.37 $\pm$ {\fontsize{6}{7.2}\selectfont 0.19} & 64.0\\
Sigm. & 907.57 $\pm$ {\fontsize{6}{7.2}\selectfont 24.04} & 64.0 & 270.55 $\pm$ {\fontsize{6}{7.2}\selectfont 6.42} & 64.0 & 0.07 $\pm$ {\fontsize{6}{7.2}\selectfont 1.14} & 64.0 & 1.46 $\pm$ {\fontsize{6}{7.2}\selectfont 0.14} & 64.0\\
Exp. & \bfseries{780.58} $\pm$ {\fontsize{6}{7.2}\selectfont 40.02} & 64.0 & 507.43 $\pm$ {\fontsize{6}{7.2}\selectfont 15.04} & 64.0 & 1.21 $\pm$ {\fontsize{6}{7.2}\selectfont 0.51} & 64.0 & 2.04 $\pm$ {\fontsize{6}{7.2}\selectfont 0.31} & 64.0\\
\hline
CR-AIS (Ours) & 788.67 $\pm$ {\fontsize{6}{7.2}\selectfont 36.25} & 39.0 & 308.82 $\pm$ {\fontsize{6}{7.2}\selectfont 6.81} & 56.0 & 0.68 $\pm$ {\fontsize{6}{7.2}\selectfont 0.31} & 58.6 & \bfseries{0.69} $\pm$ {\fontsize{6}{7.2}\selectfont 0.27} & 49.2\\
    \hline
    \end{tabular}
    \end{sc}
    \end{small}
    \end{center}
\end{table*}

\begin{table*}[h]
\caption{Absolute $\log Z_\pi$ estimation error for 512-dimensional distributions: Normal $\mathcal N(0, 0.01I)$, Gaussian mixtures with 8 components of variance 1, standard Laplace and Student-T distribution with 3 degrees of freedom with $M$ close to 64. Results are cross validated over different values of $\alpha$. Smallest error is in bold.}\label{tab:logZ}
    \begin{center}
    \begin{small}
    \begin{sc}
    \sisetup{table-align-uncertainty=true, separate-uncertainty=true,}
    \renewrobustcmd{\bfseries}{\fontseries{b}\selectfont}
    \renewrobustcmd{\boldmath}{}
    \setlength\tabcolsep{2pt}
    \begin{tabular}{|c|cc|cc|cc|cc|}
    \hline
        \multirow{2}{*}{$\beta$} & \multicolumn{2}{c|}{Normal} & \multicolumn{2}{c|}{Mixture} & \multicolumn{2}{c|}{Laplace} & \multicolumn{2}{c|}{Student-T} \\
            & Est. err. & $M$ & Est. err. & $M$ & Est. err. & $M$ & Est. err. & $M$ \\
    \hline
Ada. 0.9c & 5163.14 $\pm$ {\fontsize{6}{7.2}\selectfont 217.21} & 54.8 & 1147.02 $\pm$ {\fontsize{6}{7.2}\selectfont 136.34} & 65.6 & 8.30 $\pm$ {\fontsize{6}{7.2}\selectfont 1.34} & 67.4 & 10.39 $\pm$ {\fontsize{6}{7.2}\selectfont 2.15} & 40.8\\
Ada. 0.6c & 5181.95 $\pm$ {\fontsize{6}{7.2}\selectfont 232.09} & 53.4 & 1147.02 $\pm$ {\fontsize{6}{7.2}\selectfont 142.66} & 66.4 & 9.12 $\pm$ {\fontsize{6}{7.2}\selectfont 3.01} & 66.4 & 9.72 $\pm$ {\fontsize{6}{7.2}\selectfont 2.41} & 37.2\\
Ada. 0.7c & 5297.43 $\pm$ {\fontsize{6}{7.2}\selectfont 188.64} & 56.4 & 1170.21 $\pm$ {\fontsize{6}{7.2}\selectfont 104.98} & 66.4 & 7.01 $\pm$ {\fontsize{6}{7.2}\selectfont 3.31} & 64.0 & 10.66 $\pm$ {\fontsize{6}{7.2}\selectfont 0.83} & 39.2\\
Ada. 0.8c & 5195.95 $\pm$ {\fontsize{6}{7.2}\selectfont 105.26} & 55.8 & 1119.56 $\pm$ {\fontsize{6}{7.2}\selectfont 79.00} & 69.6 & 8.27 $\pm$ {\fontsize{6}{7.2}\selectfont 3.26} & 64.4 & 9.92 $\pm$ {\fontsize{6}{7.2}\selectfont 1.65} & 38.6\\
Ada. 0.9 & 5107.72 $\pm$ {\fontsize{6}{7.2}\selectfont 142.41} & 66.0 & 1662.46 $\pm$ {\fontsize{6}{7.2}\selectfont 94.87} & 39.0 & 9.64 $\pm$ {\fontsize{6}{7.2}\selectfont 2.18} & 40.8 & 9.97 $\pm$ {\fontsize{6}{7.2}\selectfont 1.01} & 40.0\\
Ada. 0.6 & 5203.06 $\pm$ {\fontsize{6}{7.2}\selectfont 115.42} & 65.2 & 1033.74 $\pm$ {\fontsize{6}{7.2}\selectfont 42.08} & 69.2 & 7.14 $\pm$ {\fontsize{6}{7.2}\selectfont 2.52} & 69.8 & \bfseries{7.83} $\pm$ {\fontsize{6}{7.2}\selectfont 1.26} & 69.4\\
Ada. 0.7 & 5190.67 $\pm$ {\fontsize{6}{7.2}\selectfont 112.92} & 65.2 & 1095.46 $\pm$ {\fontsize{6}{7.2}\selectfont 30.55} & 70.0 & 10.01 $\pm$ {\fontsize{6}{7.2}\selectfont 3.15} & 36.2 & 10.15 $\pm$ {\fontsize{6}{7.2}\selectfont 1.36} & 36.4\\
Ada. 0.8 & 5151.53 $\pm$ {\fontsize{6}{7.2}\selectfont 87.72} & 65.6 & 1266.31 $\pm$ {\fontsize{6}{7.2}\selectfont 67.74} & 69.8 & 9.65 $\pm$ {\fontsize{6}{7.2}\selectfont 1.79} & 38.6 & 9.94 $\pm$ {\fontsize{6}{7.2}\selectfont 1.54} & 39.4\\
Lin. & 5279.55 $\pm$ {\fontsize{6}{7.2}\selectfont 79.56} & 64.0 & 1087.38 $\pm$ {\fontsize{6}{7.2}\selectfont 22.68} & 64.0 & 9.03 $\pm$ {\fontsize{6}{7.2}\selectfont 0.77} & 64.0 & 8.16 $\pm$ {\fontsize{6}{7.2}\selectfont 0.45} & 64.0\\
Sigm. & 4757.13 $\pm$ {\fontsize{6}{7.2}\selectfont 59.34} & 64.0 & \bfseries{1022.60} $\pm$ {\fontsize{6}{7.2}\selectfont 23.55} & 64.0 & \bfseries{6.95} $\pm$ {\fontsize{6}{7.2}\selectfont 2.48} & 64.0 & 8.71 $\pm$ {\fontsize{6}{7.2}\selectfont 1.24} & 64.0\\
Exp. & 4435.05 $\pm$ {\fontsize{6}{7.2}\selectfont 110.36} & 64.0 & 1683.21 $\pm$ {\fontsize{6}{7.2}\selectfont 13.73} & 64.0 & 13.17 $\pm$ {\fontsize{6}{7.2}\selectfont 0.40} & 64.0 & 13.98 $\pm$ {\fontsize{6}{7.2}\selectfont 1.08} & 64.0\\
\hline
CR-AIS (Ours) & \bfseries{4413.95} $\pm$ {\fontsize{6}{7.2}\selectfont 95.75} & 64.8 & 1200.12 $\pm$ {\fontsize{6}{7.2}\selectfont 27.12} & 70.0 & 8.75 $\pm$ {\fontsize{6}{7.2}\selectfont 1.84} & 55.4 & 8.40 $\pm$ {\fontsize{6}{7.2}\selectfont 1.28} & 49.2\\
    \hline
    \end{tabular}
    \end{sc}
    \end{small}
    \end{center}
\end{table*}

\section{Bayesian Logistic Regression Details}\label{sec:app:bayes}

For the Bayesian logistic regression experiment we have the following hyperparameters.
We vary $\delta$ between $\{64, ..., 2048\}$ for CR-AIS and maximum step size in the range  $\{1/65536, ..., 1/2048\}$ for Adaptive AIS with $N=256$ particles. We chose $\alpha$ from a grid in $[-0.5, 2.0]$ for each value of $\delta$ and maximum step size separately using cross validation on the highest estimated log marginal likelihood. 

We use standard normal distribution for $q_0$ and 1-step HMC transitions as before with step size 0.5. Although the transition kernels are considerably simple for a problem of moderate dimensions, it is chosen to compare the adaptability of the algorithms to the posterior distribution. In fact, an annealing sequence with larger $M$ and simple HMC transitions although has the same amount of computations, is more flexible than a shorter annealing sequence with larger number of MCMC steps per transition. 

We use Adaptive AIS with ESS decrease rate of 0.5.

The adapted CR-AIS schedule for different values of $\delta$ is illustrated in \Cref{fig:pima_app}(Right) and \Cref{fig:sonar_app}(Right) for Pima and Sonar datasets. As expected, the discretization schedule has a similar pattern for different values of $\delta$. It is possible to exploit this property principally to increase the computation efficiency by interpolating the schedule obtained from tuning with a large $\delta$ and elongating the annealing sequence to reach the desired $M$ for the final estimation with negligible impact on the performance. 

\begin{figure}[h]
    \begin{center}
    \centerline{\includegraphics[width=0.45\textwidth]{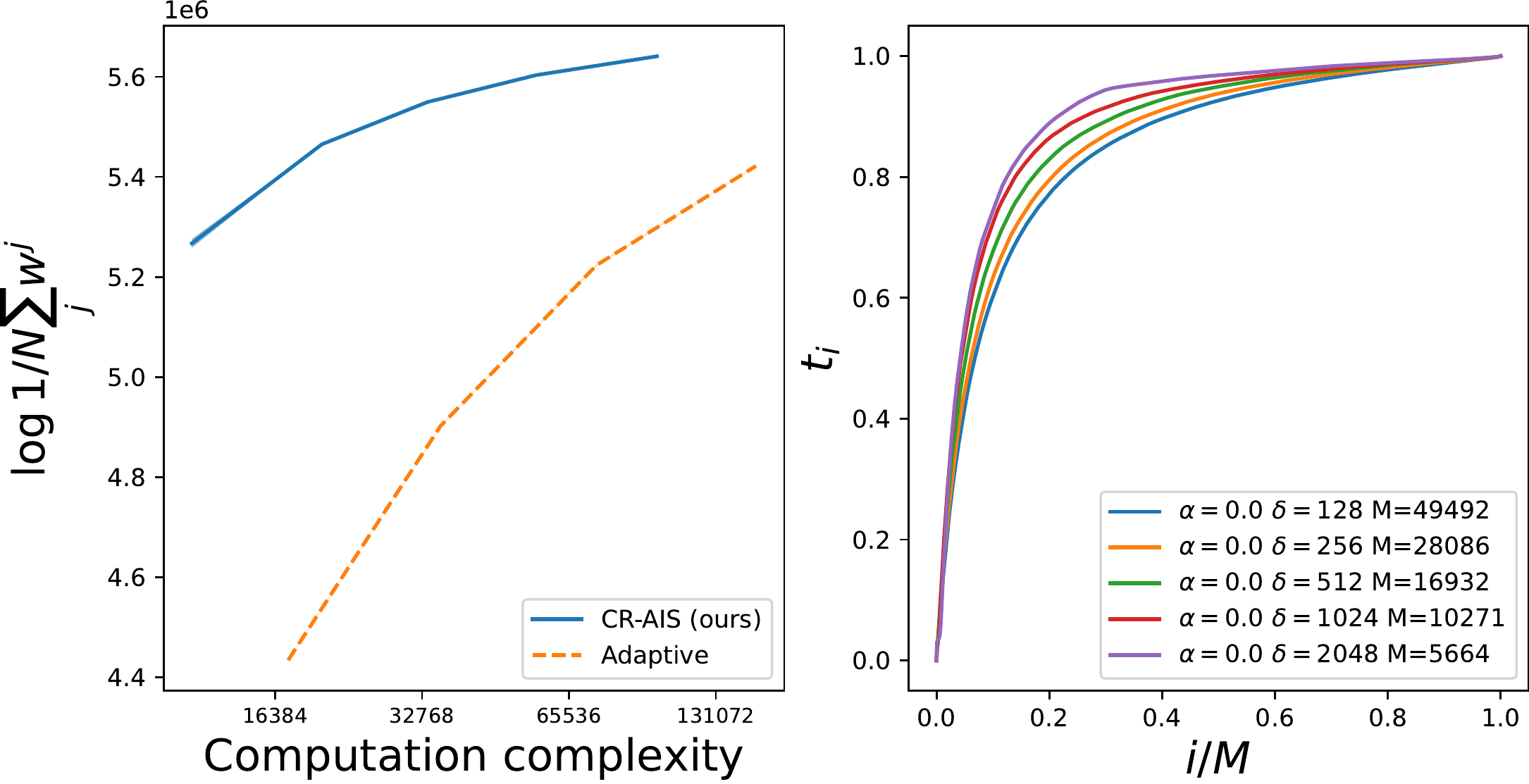}}
        \caption{(Left) Log marginal likelihood estimates vs computation complexity for CR-AIS and Adaptive AIS. (Right) CR-AIS discretization schedule for different values of $\delta$ on Bayesian logistic regression model of Pima dataset.}
    \label{fig:pima_app}
    \vskip -0.2in
    \end{center}
\end{figure}

\begin{figure}[h]
    \begin{center}
    \centerline{\includegraphics[width=0.45\textwidth]{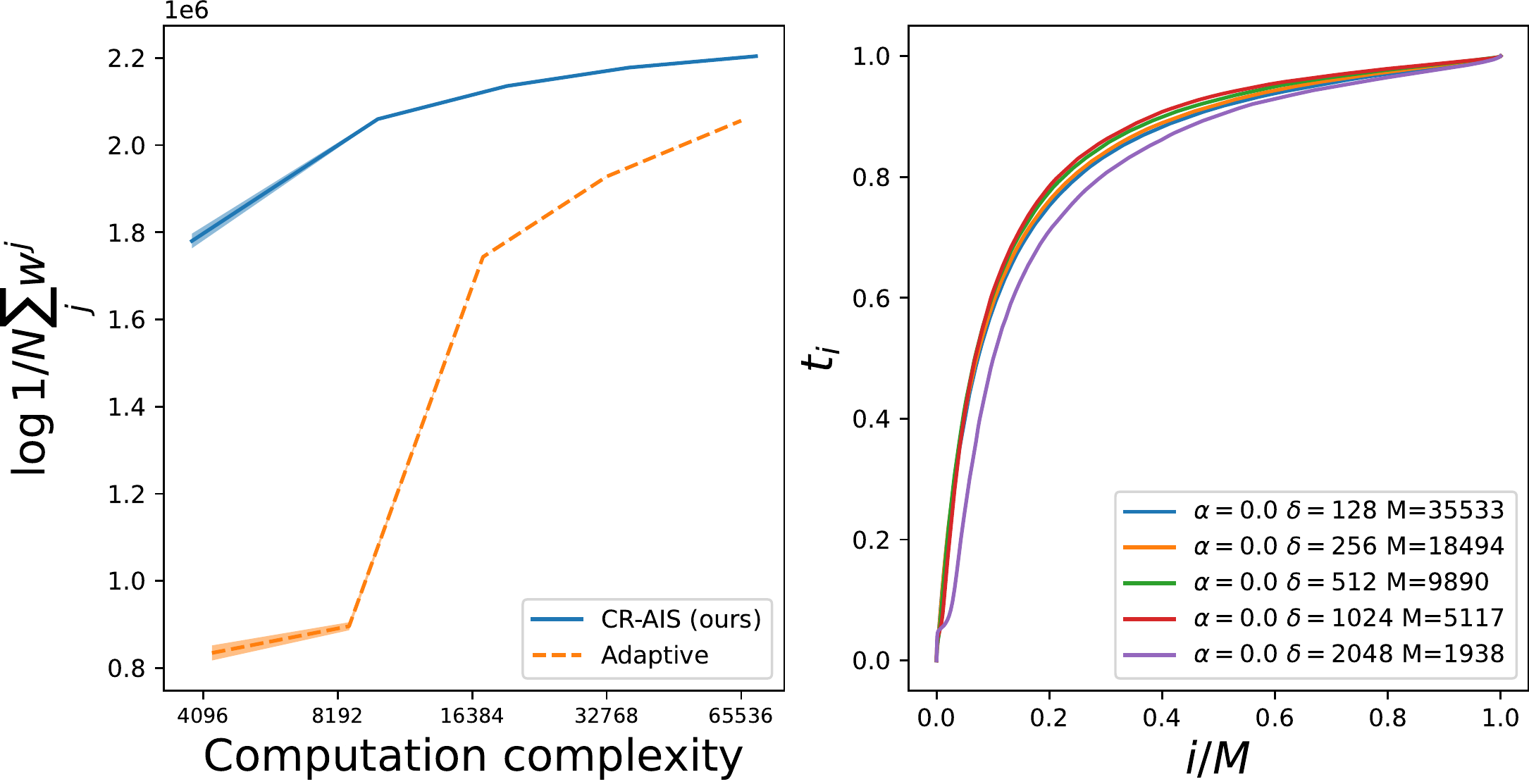}}
        \caption{(Left) Log marginal likelihood estimates vs computation complexity for CR-AIS and Adaptive AIS. (Right) CR-AIS discretization schedule for different values of $\delta$ on Bayesian logistic regression model of Sonar dataset.}
    \label{fig:sonar_app}
    \vskip -0.2in
    \end{center}
\end{figure}

\end{document}